\documentclass{article}
\usepackage{ijcai25}
\usepackage{times}  
\usepackage{soul}
\usepackage{url}  
\usepackage[hidelinks]{hyperref}
\usepackage[utf8]{inputenc}
\usepackage[small]{caption}
\usepackage{graphicx}
\usepackage[switch]{lineno}

\usepackage{algorithm}
\usepackage{algpseudocode}
\usepackage{amsmath}
\usepackage{stmaryrd}
\SetSymbolFont{stmry}{bold}{U}{stmry}{m}{n}
\usepackage{amssymb}
\usepackage{amsfonts}
\usepackage{amsthm}
\usepackage{cleveref}
\usepackage{color}
\usepackage{bussproofs}
\usepackage{optidef}
\usepackage{nicematrix}
\usepackage{booktabs}
\usepackage{physics2}
\usephysicsmodule{ab}
\usepackage{siunitx}
\algtext*{EndWhile} 
\algtext*{EndIf} 
\algtext*{EndFor} 

\newif\iffull\fulltrue
\newif\ifsubmission\submissiontrue

\newif\ifmemo\memofalse

\algnewcommand\algorithmicforeach{\textbf{for each}}
\algdef{S}[FOR]{ForEach}[1]{\algorithmicforeach\ #1\ \algorithmicdo}

\expandafter\let\expandafter\oldproof\csname\string\proof\endcsname
\let\oldendproof\endproof
\renewenvironment{proof}[1][\proofname]{%
  \oldproof[\bfseries\itshape #1] 
}{\oldendproof}


\theoremstyle{plain}
\newtheorem{thm}{Theorem}[section]

\newtheorem{prop}[thm]{Proposition}
\newtheorem{cor}[thm]{Corollary}

\theoremstyle{definition}
\newtheorem{defn}[thm]{Definition}

\newtheorem{lemma}[thm]{Lemma}

\crefname{section}{\S\!\!}{Sections}
\crefname{defn}{Def.}{Definitions}
\crefname{thm}{Thm.}{Theorems}
\crefname{lemma}{Lem.}{Lemmas}
\crefname{cor}{Cor.}{Corollaries}
\crefname{prop}{Prop.}{Propositions}
\crefname{figure}{Fig.}{Figures}
\crefname{equation}{Eq.}{Equations}
\crefname{align}{Eq.}{Equations}
\crefname{appendix}{Appendix}{Appendixes}
\crefname{algorithm}{Alg.}{Algorithms}
\crefname{table}{Tab.}{Tables}
\crefname{mini}{Prob.}{Problems}
\usepackage{mathtools}

\newcommand{\vect}[1]{\mathbf{#1}}
\DeclareMathOperator{\vectop}{vec}
\newcommand{\R}{\mathbb{R}}
\newcommand{\N}{\mathbb{N}}
\newcommand{\Rnneg}{\mathbb{R}_{\geq 0}}
\newcommand{\Rinf}{\overline{\R}}

\newcommand{\defeq}{\coloneqq}
\newcommand{\seqcomp}{\fatsemi}
\newcommand{\id}[1]{\mathrm{id}_{#1}}
\newcommand{\idcost}[1]{\mathbf{id}_{#1}}

\newcommand{\sd}[1]{\mathbb{#1}}

\newcommand{\OT}[3]{\mathrm{OT}(#1, #2, #3)}
\newcommand{\SeqOT}[4]{\mathrm{SeqOT}(#1 \seqcomp #2, #3, #4)}
\newcommand{\ParOT}[4]{\mathrm{ParOT}(#1\otimes #2, #3, #4)}
\newcommand{\addTS}{+_{\mathrm{t}}}
\newcommand{\Kron}{\otimes_{\R}}
\newcommand{\mulTS}{\cdot_{\mathrm{t}}}
\newcommand{\oOT}[1]{\mathcal{#1}}
\newcommand{\ob}[1]{\mathrm{ob}(#1)}
\newcommand{\composedLP}{\mathrm{CompLP}}
\newcommand{\monLP}{\mathrm{MonLP}}

\newcommand{\axrule}{\mathrm{Ax}}

\newcommand{\seqrule}{\seqcomp}
\newcommand{\sumrule}{\otimes}
\newcommand{\idrule}{\mathrm{ID}}
\DeclareRobustCommand{\svdots}{
  \vbox{%
    \baselineskip=0.33333\normalbaselineskip
    \lineskiplimit=0pt
    \hbox{.}\hbox{.}\hbox{.}%
    \kern-0.2\baselineskip
  }%
}

\def\Setdef#1|#2\Setdef{\left\{#1\,\;\mathstrut\vrule\,\;#2\right\}}%
\renewcommand{\th}{%
    \ifmmode
        ^\mathrm{th}%
    \else%
        \textsuperscript{th}\xspace%
    \fi%
}
\makeatletter

\newenvironment{proofs}{%
  \renewcommand{\proofname}{Proof Sketch}\proof}{\endproof}

\title{String Diagram of Optimal Transports}
\author{
    Kazuki Watanabe$^1$
    \and
    Noboru Isobe$^2$\\
    \affiliations
    $^1$National Institute of Informatics\\
    $^2$The University of Tokyo
    \emails
    kazukiwatanabe@nii.ac.jp,
    nobo0409@g.ecc.u-tokyo.ac.jp
}

\usepackage{bibentry}

\begin{document}

\maketitle

\begin{abstract}
We present a novel hierarchical framework for optimal transport (OT) using string diagrams, namely \emph{string diagrams of optimal transports.}
This framework reduces complex hierarchical OT problems to standard OT problems, allowing efficient synthesis of optimal hierarchical transportation plans.
Our approach uses algebraic compositions of cost matrices to effectively model hierarchical structures.
We also study an adversarial situation with multiple choices in the cost matrices, where we present a polynomial-time algorithm for a relaxation of the problem.
Experimental results confirm the efficiency and performance advantages of our proposed algorithm over the naive method.
\end{abstract}

\section{Introduction}
\emph{Optimal transport (OT)} is a classic problem in Operations Research, which nowadays attracts much attention from various fields. 
Kantrovich introduces the discrete OT formulated in linear programming (LP)~\cite{Kantrovich42}; see e.g.~\cite{ftml/PeyreC19}. 
Solving the discrete OT amounts to computing the minimum transportation cost, constrained by two discrete distributions. 
Typically, algorithms that solve the discrete OT also provide a \emph{transportation plan} (or \emph{coupling}) along which the transport achieves the minimum cost. 
Therefore, we can consider the discrete OT as a planning problem to find an optimal transportation plan, which has been well studied in the context of artificial intelligence; see e.g.~\cite{aw/RN2020}.  

In real-world problems, models, which are represented by cost matrices in the context of OT, often have \emph{hierarchical structures}. 
For instance, buildings or maps consist of rooms or streets, respectively, and it is natural to assume that we can rely on such domain-specific knowledge about structures. 
Indeed, hierarchical planning on such hierarchical models has been well studied in hierarchical reinforcement learning~\cite{deds/BartoM03,csur/PateriaSTQ21}.

\emph{String diagrams} (e.g.~\cite{mac2013categories}) are a graphical language that can naturally capture such hierarchical structures with two algebraic compositions: a \emph{sequential composition} $\seqcomp$ and a \emph{parallel composition} $\otimes$. 
In particular, \emph{string diagrams of Markov decision processes (MDPs)}~\cite{cav/WatanabeEAH23,tacas/WatanabeVHRJ24} exploit such hierarchical structures on MDPs to enhance the performance of hierarchical planning on MDPs. 

In this paper, we propose a hierarchical framework of OTs that models such hierarchical structures in string diagrams. 
Specifically, we formulate \emph{string diagrams of OTs} equipped with the two algebraic operations over \emph{open OTs}, which is illustrated in the following example:
\begin{figure}[t]
\centerline{\includegraphics[width=0.9\linewidth]{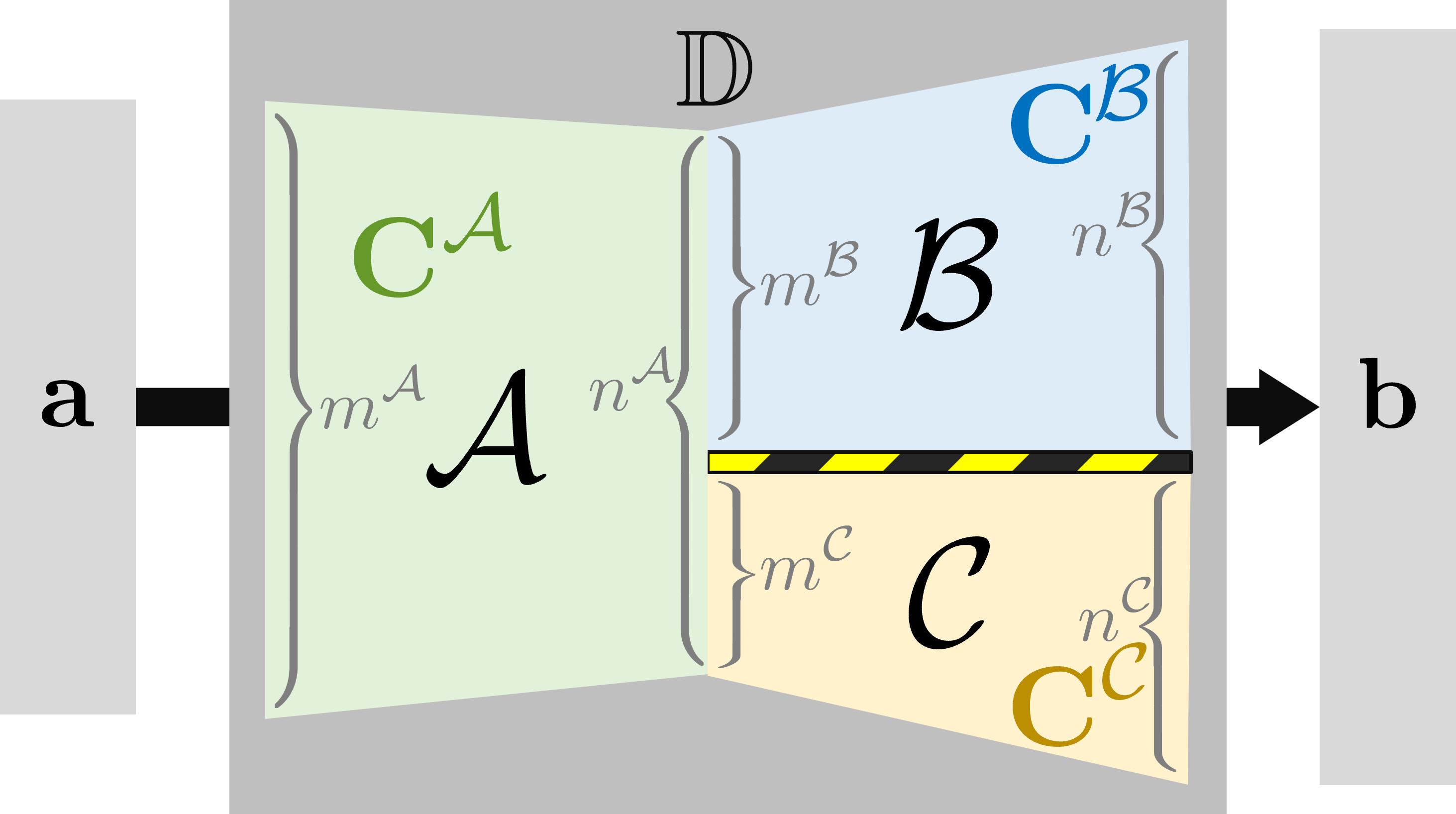}}
\caption{A string diagram $\sd{D}\defeq \oOT{A}\seqcomp (\oOT{B}\otimes\oOT{C})$ with distributions $\vect{a}, \vect{b}$.}
\label{fig:example}
\end{figure}
\paragraph{Example.}
Let $\vect{a}$ and $\vect{b}$ be two distributions. We present a string diagram $\sd{D} \defeq \oOT{A}\seqcomp (\oOT{B}\otimes \oOT{C})$ with open OTs $\oOT{A},\oOT{B}, \oOT{C}$ as an example in~\cref{fig:example}. Each open OT consists of a cost matrix, e.g., the cost matrix of $\oOT{A}$ is $\vect{C}^{\oOT{A}}\in \Rnneg^{m^{\oOT{A}}\times n^{\oOT{A}}}$, where $m^{\oOT{
A}}$ and $n^{\oOT{
A}}$ denote the size of its \emph{connections} on the left and right, respectively. 
We compose open OTs with their connections. 
The connections $m^{\oOT{B}\otimes\oOT{C}}$ and $n^{\oOT{B}\otimes\oOT{C}}$  of the parallel composition $\oOT{B}\otimes \oOT{C}$ are the sum of the connections of the components, that is, $m^{\oOT{B}\otimes\oOT{C}} = m^{\oOT{
B}} + m^{\oOT{C}}$ and $n^{\oOT{B}\otimes\oOT{C}} = n^{\oOT{
B}} + n^{\oOT{C}}$, respectively. 
Then, the connections on the left and right of the sequential composition $\oOT{A}\seqcomp (\oOT{B}\otimes \oOT{C})$ are $m^{\oOT{A
}}$ and $n^{\oOT{B}\otimes\oOT{C}}$. 
Importantly, the connections between $\oOT{A}$ and $(\oOT{B}\otimes \oOT{C})$ must be consistent by the sequential composition $\seqcomp$, that is, $n^{\oOT{A}} = m^{\oOT{B}\otimes\oOT{C}}$. 
The optimal transport problem on $\sd{D}$ is the minimization problem given by 
\[
\min_{\vect{P}^{\oOT{A}}, \vect{P}^{\oOT{B}},\vect{P}^{\oOT{C}}} 
\sum_{\star \in \{\oOT{A}, \oOT{B}, \oOT{C}\}} \langle \vect{C}^{\star}, \vect{P}^{\star}\rangle,
\]
subject to the following three conditions: (i)
\emph{hierarchical transportation plans} $\vect{P}^{\oOT{A}}, \vect{P}^{\oOT{B}}, \vect{P}^{\oOT{C}}$ are non-negative matrices; (ii) the plans are consistent with $\vect{a}$ and $\vect{b}$, that is, 
\begin{align*}
\vect{a} &= \vect{P}^{\oOT{A}}\vect{1}_{n^{\oOT{A}}},
&&\vect{b} = \big((\vect{P}^{\oOT{B}})^{\top}\vect{1}_{m^{\oOT{B}}}, \,(\vect{P}^{\oOT{C}})^{\top}\vect{1}_{m^{\oOT{C}}}\big)
\end{align*}
and (iii) the plans are consistent with the sequential composition $\seqcomp$, that is, \[(\vect{P}^{\oOT{A}})^{\top}\vect{1}_{m^{\oOT{A}}} = (\vect{P}^{\oOT{B}}\vect{1}_{n^{\oOT{B}}},\, \vect{P}^{\oOT{C}}\vect{1}_{n^{\oOT{C}}}).\]

In general, our hierarchical planning problem on string diagrams of OTs can be formulated as follows:
{\it What is an optimal hierarchical transportation plan in the structure, where each component has its cost matrix and components are connected by algebraic operations through connections?}

A challenge of this hierarchical planning on a string diagram of OTs is that the size of its hierarchical transportation plan can increase dramatically due to the compositions. 
Empirically, increasing the size of transportation plans worsens performance, as we can easily imagine. 
In this paper, we propose a novel algorithm that computes an optimal hierarchical transportation plan by \emph{composing} OTs. The crux of our approach is the exploitation of the algebraic structure of cost matrices.  The algebraic structure naturally leads us to define a novel reduction from a string diagram of OTs to a monolithic (or standard) OT. Our algorithm proceeds by (i) reducing a given string diagram to an OT; and (ii) constructing an optimal hierarchical transportation plan of the string diagram from an optimal transportation plan of the OT.



We further extend string diagrams of OTs by allowing adversarial choices in cost structures, which we call \emph{string diagrams of OTs with choices of cost matrices}, as a max-min problem.    
A naive algorithm for this problem can be an enumeration algorithm of the choice of cost matrices. The number of such choices is exponential in the number of compositions, hence the enumeration algorithm is infeasible. 
We propose its relaxation problem, and reduce the relaxation problem to LP, which is solvable in polynomial time. Our reduction is indeed very simple: we can regard the problem as a piecewise-linear minimization through the \emph{mini-max principle} (e.g.~\cite{mazalov2014mathematical}), and thus we can translate the problem into LP. 


We evaluate the performance of our algorithm for string diagrams of OTs using realistic benchmarks.
Specifically, we demonstrate that our proposed algorithm for string diagrams of OTs shows its performance advantage over the naive LP approach; we note that, to the best of our knowledge, LP is the only existing algorithm that can directly solve string diagrams of OTs. The benchmarks for hierarchical planning are inspired by the benchmarks for hierarhical MDPs~\cite{tacas/WatanabeVHRJ24}. 

\paragraph{Contribution.}
Our contributions are as follows: 
\begin{itemize}
    \item We introduce two compositions of OTs, the sequential composition $\seqcomp$, and the parallel composition $\otimes$ (in~\cref{sec:composedOT}).
    \item We study the dual problems of composed OTs, which naturally lead to the notion of compositions over cost matrices (in~\cref{sec:duality}).
    \item We formulate (aligned) string diagrams of OTs as a hierarchical framework of OTs, 
    and present our novel reduction to monolithic (or standard) OTs (in~\cref{sec:sdOT}).
    \item We propose our hierarchical planning algorithm for string diagrams of OTs (in~\cref{sec:reduction}).
    \item We further study string diagrams of OTs with adversarial choices of cost matrices, and show its relaxation problem is solvable in polynomial time (in~\cref{sec:choiceCostMats}).
    \item We evaluate the performance of our algorithm for string diagrams of OTs through experiments (in~\cref{sec:experiments}). 
\end{itemize}

\paragraph{Notation.}
We write  $\Rnneg$ for the set of non-negative real numbers, and  $\Rinf$ for the set of real numbers and the (positive) infinity $\infty$, 
where the multiplication $\infty \cdot 0\defeq 0$. 
For matrices $\vect{X}, \vect{Y}\in \Rinf^{m\times n}$,  we write $\langle \vect{X}, \vect{Y} \rangle $ for $\sum_{i\in [m], j\in [n]} X_{ij}Y_{ij}$, where 
the sets $[m], [n]$ are given by $[m]\defeq \{1, \dots, m\}$ and $[n]\defeq \{1, \dots, n\}$, respectively. 
We write $\Delta^{m}$ for the probability simplex $\Delta^{m} \defeq \{\vect{a}=(a_i)_{i=1}^m \in \Rnneg^{m} \mid \sum_{i\in [m]} a_i = 1\}$. 
We write $\vect{1}_m\in \R^{m}$ for the all-ones vector.

\section{Compositions of Optimal Transports}
\label{sec:composedOT}
In this section, we introduce two compositions of OTs, namely the \emph{sequential composition} $\seqcomp$ and the \emph{parallel composition} $\otimes$. We start by recalling the standard discrete OT.
\begin{defn}[OT]
Given two distributions $\vect{a} \in \Delta^m$ and $\vect{b} \in \Delta^n$, and a cost matrix $\vect{C}\in \Rnneg^{m\times n}$, 
the $\OT{\vect{C}}{\vect{a}}{\vect{b}}$  is the minimization problem defined by 
\begin{align}
    &\min_{\vect{P}\in \Rnneg^{m\times n}} \langle \vect{C}, \vect{P} \rangle\text{ s.t. } \vect{P}\vect{1}_{n} = \vect{a},\text{ and } \vect{P}^\top\vect{1}_{m} = \vect{b}.
\end{align}
\end{defn}

We then extend OTs to a compositional framework of OTs.
The basic components in the compositional framework are \emph{open OTs}.
Formally, an \emph{open OT (oOT)} $\oOT{A}$ is a tuple $(m, n, \vect{C}^{\oOT{A}})$ of the sizes $m, n$ of the connections on the left and right, respectively, and the cost matrix $\vect{C}^{\oOT{A}}\in\Rnneg^{m\times n}$.
We often write $\oOT{A}\colon m\rightarrow n$ with explicitly denoting the size of connections.  
Importantly, oOTs do not have distributions on connections. 

\begin{defn}
[sequentially composed OT]
Given two oOTs $\oOT{A}\defeq (m, l, \vect{C}^{\oOT{A}})$ and $\oOT{B}\defeq (l, n, \vect{C}^{\oOT{B}})$, 
and two distributions $\vect{a} \in \Delta^m$ and $\vect{b} \in \Delta^n$,
we define the \emph{sequentially composed OT} $\SeqOT{\oOT{A}}{\oOT{B}}{\vect{a}}{\vect{b}}$ as follows:
\begin{equation}
    \label{align:seqOT}
    \hspace{-9pt}\begin{array}{c@{}c}
            &\min\limits_{\vect{P}^{\oOT{A}}\in \Rnneg^{m\times l},\,\vect{P}^{\oOT{B}}\in \Rnneg^{l\times n}} \langle \vect{C}^{\oOT{A}}, \vect{P}^{\oOT{A}} \rangle +  \langle \vect{C}^{\oOT{B}}, \vect{P}^{\oOT{B}}\rangle, \\
        & \text{ s.t. }\,\, \begin{aligned}
                             \vect{P}^{\oOT{A}}\vect{1}_{l} = \vect{a},\,(\vect{P}^{\oOT{B}})^\top\vect{1}_{l} = \vect{b},\ (\vect{P}^{\oOT{A}})^\top\vect{1}_{m} = \vect{P}^{\oOT{B}}\vect{1}_{n}
                        \end{aligned}
                        .
    \end{array}
\end{equation}
\end{defn}
An intuition of the sequential composition is the following: the two oOTs $\oOT{A}\colon m\rightarrow l$ and $\oOT{B}\colon l\rightarrow n$ are connected by the connections $l$. 
We shall minimize the transportation cost by optimally choosing a transportation plan while satisfying the consistencies with 
(i) the two distributions $\vect{a}$ and $\vect{b}$ (the first and second conditions of the constraints in~\cref{align:seqOT}); and
(ii) the connections $l$ (the third condition in~\cref{align:seqOT}). 
\begin{defn}[parallelly composed OT]
Given two oOTs $\oOT{A}\defeq (m, n, \vect{C}^{\oOT{A}})$ and $\oOT{B}\defeq (k, l, \vect{C}^{\oOT{B}})$, 
and two distributions $\vect{a} \in \Delta^{m+k}$ and $\vect{b} \in \Delta^{n+l}$,
we define the \emph{parallelly composed OT} $\ParOT{\oOT{A}}{\oOT{B}}{\vect{a}}{\vect{b}}$ as follows:
\begin{equation}
    \hspace{-15pt}\begin{array}{cc}
        &\min\limits_{\vect{P}^{\oOT{A}} \in  \Rnneg^{m\times n},\,
        \vect{P}^{\oOT{B}} \in \Rnneg^{k\times l}} \langle \vect{C}^{\oOT{A}}, \vect{P}^{\oOT{A}} \rangle +  \langle \vect{C}^{\oOT{B}}, \vect{P}^{\oOT{B}}\rangle, \\
    &\text{ s.t. }
    \left\{
    \begin{aligned}
                         &\vect{P}^{\oOT{A}}\vect{1}_{n} = (a_i)_{i=1}^m, (\vect{P}^{\oOT{A}})^{\top}\vect{1}_{m} = (b_j)_{j=1}^n,\\
                        &\vect{P}^{\oOT{B}}\vect{1}_{l} = (a_{i+m})_{i=1}^{k}, (\vect{P}^{\oOT{B}})^{\top}\vect{1}_{k} = (b_{j+n})_{j=1}^{l},
                    \end{aligned}
                    \right.
\end{array}
    \label{prob:ParOT}
\end{equation}
where $\vect{P}^{\oOT{A}}\in  \Rnneg^{m\times n}$ and $\vect{P}^{\oOT{B}}\in \Rnneg^{k\times l}$. 
\end{defn}
In the parallelly composed OT $\ParOT{\oOT{A}}{\oOT{B}}{\vect{a}}{\vect{b}}$, unlike the sequentially composed OTs, 
there are no interactions between the two oOTs $\oOT{A}$ and $\oOT{B}$.
Therefore, we can solve $\ParOT{\oOT{A}}{\oOT{B}}{\vect{a}}{\vect{b}}$ by solving the two unbalanced OTs independently---a generalization of OTs that does not require the sum of distributions to be $1$ (e.g.~\cite{chizat2018unbalanced}). 
It is easy to see that the parallelly composed OT has a solution if the equation $\sum^{m}_{i=1} a_i = \sum^{n}_{j=1}  b_j$ holds.




\section{ Duality}
\label{sec:duality}
In this section, we show the  dual problems of sequentially composed OTs and parallelly composed OTs, 
which shed light on an algebraic structure in OTs. 

\subsection{Dual Problem of Sequentially Composed OTs}
Given two oOTs $\oOT{A}\defeq (m, l, \vect{C}^{\oOT{A}})$ and $\oOT{B}\defeq (l, n, \vect{C}^{\oOT{B}})$, 
and two distributions $\vect{a} \in \Delta^m$ and $\vect{b} \in \Delta^n$,
the  dual problem of $\SeqOT{\oOT{A}}{\oOT{B}}{\vect{a}}{\vect{b}}$ is given as follows. 
\begin{defn}
    \label{def:dualSeq}
    The \emph{dual problem of $\SeqOT{\oOT{A}}{\oOT{B}}{\vect{a}}{\vect{b}}$} is given by 
    \begin{gather}
         \sup_{\vect{f}\in \R^{m},\, \vect{g}\in \R^{n}} \sum^{m}_{i=1} f_i\cdot a_i + \sum^{n}_{j=1} g_j\cdot b_j\label{align:constDualSeq}\\
         \text{ s.t. }  f_i + g_j \leq C^{\oOT{A}}_{ik} + C^{\oOT{B}}_{kj},
        \label{eq:only_fg}
    \end{gather}
     for any $i\in [m], k\in [l],\text{ and }j\in [n]$.
\end{defn}
Since $\SeqOT{\oOT{A}}{\oOT{B}}{\vect{a}}{\vect{b}}$ has a feasible solution and all constraints are linear, 
the strong duality theorem holds.
The proof is followed by the standard theory of LP (e.g.~\cite[Theorem 4.4]{bertsimas1997introduction}); see the technical appendix ($\S$A) for the proof. 
\begin{prop}[strong duality]
    \label{prop:strong_duality_Seq}
    The  dual problem defined in~\cref{def:dualSeq} is equivalent to $\SeqOT{\oOT{A}}{\oOT{B}}{\vect{a}}{\vect{b}}$. 
\end{prop}

The constraints shown in~\cref{eq:only_fg} are equivalent to the following inequalities
\begin{align}
    \label{align:euivConstDualSeq}
    f_i + g_j \leq \min_{k\in [l]}\left\{ C^{\oOT{A}}_{ik} + C^{\oOT{B}}_{kj}\right\},
\end{align}
for any $i\in [m], \text{ and }j\in [n]$. 
In fact, the constraints shown in~\cref{align:euivConstDualSeq} are exactly 
those of the dual problem of the $\OT{\vect{C}^{\oOT{A}}\seqcomp \vect{C}^{\oOT{B}}}{\vect{a}}{\vect{b}}$, where 
the \emph{sequential composition} $\vect{C}^{\oOT{A}}\seqcomp \vect{C}^{\oOT{B}}$ of the cost matrices $\vect{C}^{\oOT{A}}$ and $\vect{C}^{\oOT{B}}$ is the $m\times n$ matrix whose component is given by  $(\vect{C}^{\oOT{A}}\seqcomp \vect{C}^{\oOT{B}})_{ij}\defeq \min_{k\in [l]}\{C^{\oOT{A}}_{ik} + C^{\oOT{B}}_{kj}\}$.
An informal interpretation of this fact can be the following: the cost matrix $\vect{C}^{\oOT{A}}\seqcomp \vect{C}^{\oOT{B}}$ collects minimum ``paths'', corresponding to optimal choices of $k$, along which the cost is $C^{\oOT{A}}_{ik} + C^{\oOT{B}}_{kj}$, for each $i$ and $j$. 


We can also think of the sequential composition as the multiplication of matrices over the \emph{min-tropical semiring} $M\defeq (\Rinf, \addTS, \mulTS)$: 
 the addition $r_1 \addTS r_2 $ is defined by the min operator $ r_1 \addTS r_2\defeq \min\{r_1, r_2\}$, 
 and the multiplication $r_1 \mulTS r_2 $ is defined by the sum operator $ r_1 \mulTS r_2\defeq r_1 + r_2$.

\begin{prop}
    \label{prop:eq_seqOT}
   The sequentially composed OT $\SeqOT{\oOT{A}}{\oOT{B}}{\vect{a}}{\vect{b}}$ is equivalent to the $\OT{\vect{C}^{\oOT{A}}\seqcomp \vect{C}^{\oOT{B}}}{\vect{a}}{\vect{b}}$. 
\end{prop}

Note that the equivalence is w.r.t.~minimum transportation costs, 
not w.r.t. optimal transportation plans. 

\subsection{Dual Problem of Parallelly Composed OT}
Given two oOTs $\oOT{A}\defeq (m, n, \vect{C}^{\oOT{A}})$ and $\oOT{B}\defeq (l, k, \vect{C}^{\oOT{B}})$, 
and two distributions $\vect{a} \in \Delta^{m+l}$ and $\vect{b} \in \Delta^{n+k}$, 
we show the dual problem of the parallelly composed OT $\ParOT{\oOT{A}}{\oOT{B}}{\vect{a}}{\vect{b}}$ as follows: 
\begin{defn}
    \label{def:dualPar}
    The \emph{dual problem of  $\ParOT{\oOT{A}}{\oOT{B}}{\vect{a}}{\vect{b}}$} is given by
    \begin{equation}
    \hspace{-5pt}\begin{array}{ll}
        &\displaystyle\sup\limits_{\vect{f}\in \R^{m+l},\, \vect{g}\in \R^{n+k}} \sum^{m+l}_{i=1} f_i\cdot a_i + \sum^{n+k}_{j=1} g_j\cdot b_j,\\
        \label{align:constDualPar}
        &\text{ s.t. }\quad  
        \begin{aligned}
             f_i + g_j &\leq C^{\oOT{A}}_{i j} &&(\forall i\in [m],\, j\in [l]),\\
            f_{i+m} + g_{j+l} &\leq C^{\oOT{B}}_{i j} &&(\forall i\in [n],\, j\in [k]).
       \end{aligned}
    \end{array}
    \end{equation}
\end{defn}
Under the balance condition $\sum^{m}_{i=1} a_i = \sum^{n}_{j=1}  b_j$ (otherwise $\ParOT{\oOT{A}}{\oOT{B}}{\vect{a}}{\vect{b}}$ may not have a feasible solution), the strong duality theorem holds.
See the technical appendix ($\S$A) for the proof. 

\begin{prop}[strong duality]
    \label{prop:strong_duality_Par}
    Assume that $\sum^{m}_{i=1} a_i = \sum^{n}_{j=1}  b_j$. 
    The  dual problem defined in~\cref{def:dualPar} is equivalent to $\ParOT{\oOT{A}}{\oOT{B}}{\vect{a}}{\vect{b}}$.
\end{prop}

Similar to the sequentially composed OTs, the constraints shown in~\cref{align:constDualPar} also have a compositional interpretation. 
    Given two cost matrices $\vect{C}^{\oOT{A}}\in \R^{m\times n}$ and $\vect{C}^{\oOT{B}}\in \R^{l\times k}$, 
    the \emph{parallel composition} $\vect{C}^{\oOT{A}}\otimes \vect{C}^{\oOT{B}}$ is the cost matrix given by
    \[
        \vect{C}^{\oOT{A}}\otimes \vect{C}^{\oOT{B}} \defeq \begin{pmatrix}
            \vect{C}^{\oOT{A}} & \infty\\
            \infty   & \vect{C}^{\oOT{B}}\\ 
        \end{pmatrix}
        \in\Rinf^{(m+l)\times(n+k)}.
    \]
    Formally, it is given by 
    (i) $(\vect{C}^{\oOT{A}}\otimes \vect{C}^{\oOT{B}})_{ij}\defeq C^{\oOT{A}}_{ij}$ if $i\in [m]$ and $j\in [n]$, 
    (ii) $(\vect{C}^{\oOT{A}}\otimes \vect{C}^{\oOT{B}})_{ij}\defeq C^{\oOT{B}}_{(i-m)(j-n)}$ if $i\in [m+1, m+l]$ and $j\in [n+1, n+k]$, and 
    (iii) $(\vect{C}^{\oOT{A}}\otimes \vect{C}^{\oOT{B}})_{ij}\defeq \infty$ otherwise. 

 The intuition of the parallel composition should be obvious---it imposes the $\infty$ penalties in order to prevent from jumping to the other oOT. 

 \begin{prop}
    \label{prop:ParOT_CotimesD}
    Assume that $\sum^{m}_{i=1} a_i = \sum^{n}_{j=1}  b_j$. 
    The prallelly composed OT $\ParOT{\oOT{A}}{\oOT{B}}{\vect{a}}{\vect{b}}$ is equivalent to the $\OT{\vect{C}^{\oOT{A}}\otimes \vect{C}^{\oOT{B}}}{\vect{a}}{\vect{b}}$. 
\end{prop}



\section{Aligned String Diagram of OTs}
\label{sec:sdOT}
In this section, we introduce aligned string diagrams of OTs by extending the sequentially and parallelly composed OTs. 
We exploit the underlying algebraic structure in cost matrices, which gives a simple formulation of the dual problem. 
Based on the dual problem, we introduce our novel reduction from string diagrams of OTs to (monolithic) OTs; \cref{fig:overview} illustrates the reduction with an example.

\begin{figure}[t]
    \begin{center}
        \includegraphics[width=\linewidth]{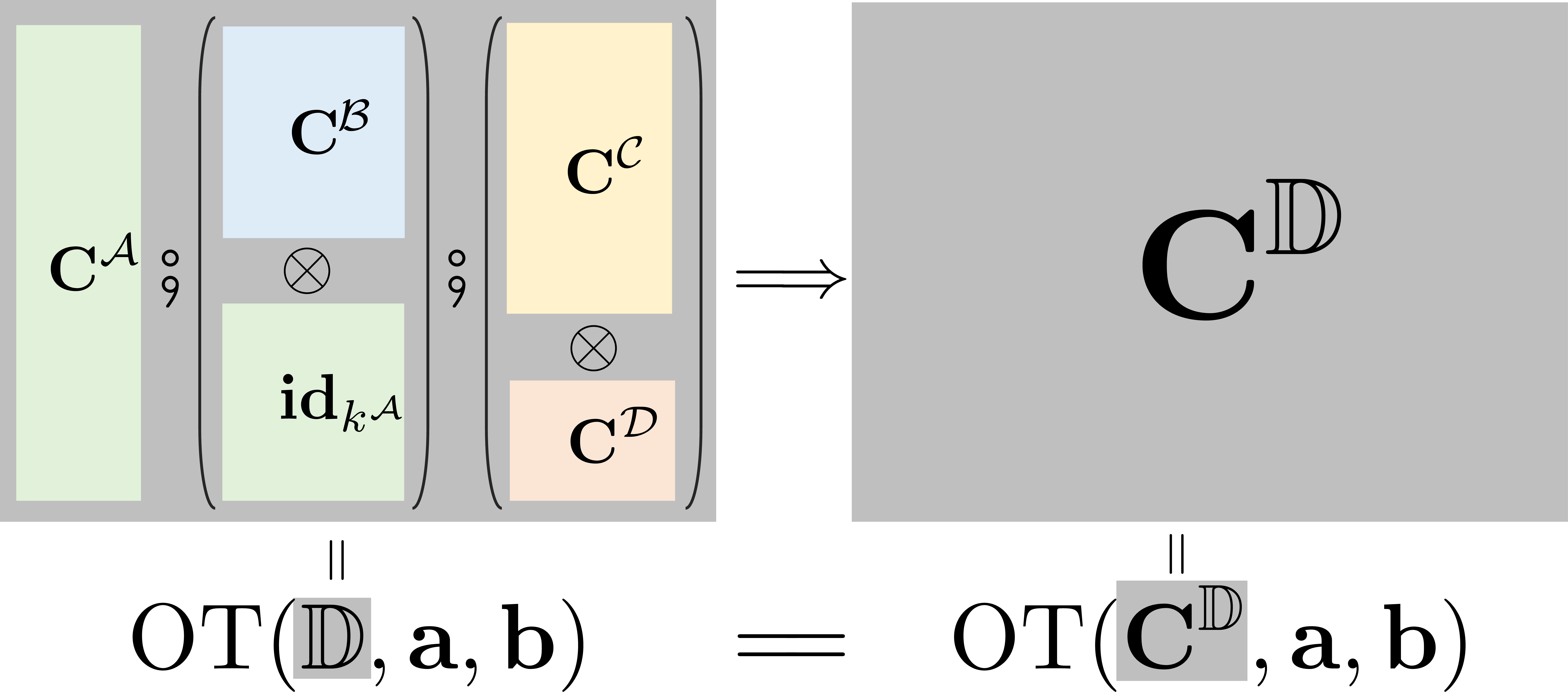}
    \end{center}
    \caption[]{
    An example of our reduction given by~\cref{cor:correctRD}.
        The string diagram $\sd{D}$ in the above figure
        is formally given by $\sd{D}\defeq \oOT{A}\seqcomp (\oOT{B}\otimes \id{k^{\oOT{A}}})\seqcomp (\oOT{C}\otimes \oOT{D})$ 
        with the sequential composition $\seqcomp$, the parallel composition $\otimes$, and the identity $\id{k^{\oOT{A}}}$.
        Similarly, the cost matrix $\vect{C}^{\sd{D}}$ on the right is given by $\vect{C}^{\sd{D}}\defeq \vect{C}^{\oOT{A}}\seqcomp (\vect{C}^{\oOT{B}}\otimes \idcost{k^{\oOT{A}}})\seqcomp (\vect{C}^{\oOT{C}}\otimes \vect{C}^{\oOT{D}})$ with
        the compositions $\seqcomp,\otimes$, and the idenity cost matrix $\idcost{k^{\oOT{A}}}$. 
    }
    \label{fig:overview} 
\end{figure}

\subsection{Aligned String Diagram of OTs}
We introduce \emph{aligned string diagrams of OTs}, which are a subclass of string diagrams of OTs defined later in \cref{sec:reduction}.
The diagram shown on the left of~\cref{fig:overview} is an aligned string diagram of OTs $ \oOT{A}\seqcomp (\oOT{B}\otimes \id{k^{\oOT{A}}})\seqcomp (\oOT{C}\otimes \oOT{D})$, where 
the oOT $\id{k^{\oOT{A}}}$ is given by $ (k^{\oOT{A}},k^{\oOT{A}},\idcost{k^{\oOT{A}}})$ with the \emph{identity cost matrix} $\idcost{k^{\oOT{A}}}$: (i) $(\idcost{k^{\oOT{A}}})_{ij}\defeq 0 $ if $i = j$, and (ii) $(\idcost{k^{\oOT{A}}})_{ij}\defeq \infty $ if $i \not= j$. 
\begin{defn}
    \label{def:aligned_string_diagram}
    Let $H$ and $l_i$ ($i\in[H]$) be positive integers, $\oOT{A}\defeq (m^{\oOT{A}},n^{\oOT{A}}, \vect{C}^{\oOT{A}})$ and 
    $\oOT{B}_{ji}\defeq (m^{\oOT{B}_{ji}},n^{\oOT{B}_{ji}}, \vect{C}^{\oOT{B}_{ji}})$ $(j\in[l_i],i\in[H])$ be oOTs. 
    An \emph{aligned string diagram of OTs $\sd{D}$ of $\oOT{A}$ and $(\oOT{B}_{ji})_{j,i}$} is a formula of oOTs given by 
    \begin{align}
        \label{eq:aligned_string_diagram}
        \oOT{A}\seqcomp (\oOT{B}_{11}\otimes \dots \otimes \oOT{B}_{1l_1})\seqcomp \dots \seqcomp (\oOT{B}_{H1}\otimes \dots \otimes \oOT{B}_{Hl_H})
    \end{align}
    such that the compositions are consistent with connections,  
    \begin{align*}
        n^{\oOT{A}} &= \sum_{i\in [l_1]} m^{\oOT{B}_{1i}},\, &\sum_{i\in [l_j]} n^{\oOT{B}_{ji}} = \sum_{i\in [l_{j+1}]} m^{\oOT{B}_{j+1i}},
    \end{align*}
    for each $j\in [H-1]$,
    Here we assume that (i) $\vect{C}^{\oOT{A}}\in  \R^{m^{\oOT{A}}\times n^{\oOT{A}}}$; and (ii) $\vect{C}^{\oOT{B}_{ji}}\in  \R^{m^{\oOT{B}_{ji}}\times n^{\oOT{B}_{ji}}}$ or $\vect{C}^{\oOT{B}_{ji}} = \idcost{m^{\oOT{B}_{ji}}}$ with $m^{\oOT{B}_{ji}} = n^{\oOT{B}_{ji}}$ for any $j$ and $i$.
\end{defn}
It is important to note that we do not allow the leftmost oOT $\oOT{A}$ to be replaced by $(\oOT{B}_{01}\otimes \dots \otimes \oOT{B}_{0l_0})$ to ensure the existence of a feasible solution. 
In addition, we require that there are no \emph{deadends} in $\sd{D}$, that is, $m^{\oOT{A}}, n^{\oOT{A}}, m^{\oOT{B}_{ji}}, n^{\oOT{B}_{ji}}\geq 1$ for any $j, i$. 
In~\cref{sec:reduction}, we introduce string diagrams of OTs that allow, for example, a formula $\oOT{A}\seqcomp \big((\oOT{B}\seqcomp \oOT{C}) \otimes \oOT{D}\big)$.


We then define our OT on an aligned string diagram $\sd{D}$. 
\begin{defn}
    \label{def:SDOT}
    Let $\vect{a} \in \Delta^{m}$ and $\vect{b} \in \Delta^{n}$ be two distributions, where $n\defeq \sum_{i\in [l_H]} n^{\oOT{B}_{Hi}}$. 
    The \emph{$\OT{\sd{D}}{\vect{a}}{\vect{b}}$ on an aligned string diagram $\sd{D}$ of $\oOT{A}$ and $(\oOT{B}_{ji})_{j,i}$} is given by 
    \begin{equation*}
        \min\limits_{\vect{P}^{\oOT{A}},\, (\vect{P}^{\oOT{B}_{ji}})_{j,i}}  \langle \vect{C}^{\oOT{A}}, \vect{P}^{\oOT{A}} \rangle +  \sum\limits_{j,i}\langle \vect{C}^{\oOT{B}_{ji}}, \vect{P}^{\oOT{B}_{ji}}\rangle,
    \end{equation*} 
    where $\vect{P}^{\oOT{A}}\in \Rnneg^{m^{\oOT{A}}\times n^{\oOT{A}}}$ and $\vect{P}^{\oOT{B}_{ji}}\in \Rnneg^{m^{\oOT{B}_{ji}}\times n^{\oOT{B}_{ji}}}$, subject to the following constraints:
    \begin{enumerate}
        \item $\vect{P}^{\oOT{A}}\vect{1}_{n^{\oOT{A}}} = \vect{a}$, and $(\vect{P}^{\oOT{B}_{H}})^{\top}\vect{1}_{m^{\oOT{B}_{H}}} = \vect{b}$,
        \item $(\vect{P}^{\oOT{A}})^{\top}\vect{1}_{m^{\oOT{A}}} = (\vect{P}^{\oOT{B}_1})\vect{1}_{n^{\oOT{B}_1}}$,
        \item $(\vect{P}^{\oOT{B}_{k}})^{\top}\vect{1}_{m^{\oOT{B}_{k}}} = (\vect{P}^{\oOT{B}_{k+1}})\vect{1}_{n^{\oOT{B}_{k+1}}}$,
    \end{enumerate}
    for each $k\in[H-1]$.
    Here, the matrix $\vect{P}^{\oOT{B}_{k}}$ is given by 
    \[
        \vect{P}^{\oOT{B}_{k}}\defeq 
        \begin{pNiceMatrix}[nullify-dots]
             \vect{P}^{\oOT{B}_{k1}} & & \Block[c]{1-1}<\Huge>{{0}} \\
             & \Ddots  &  \\
             \Block[c]{1-1}<\Huge>{{0}} &  &   \vect{P}^{\oOT{B}_{kl_{k}}}\\
        \end{pNiceMatrix}
         \in\Rnneg^{m^{\oOT{B}_{k}}\times n^{\oOT{B}_{k}}},
    \]
    and $m^{\oOT{B}_{k}} = \sum_{i\in [l_{k}]} m^{\oOT{B}_{ki}}$ and $n^{\oOT{B}_{k}} = \sum_{i\in [l_{k}]} n^{\oOT{B}_{ki}}$, for each $k\in [H]$. 
\end{defn}
Each oOT in a given aligned string diagram $\sd{D}$ has its cost matrix $\vect{C}$ and its transportation plan $\vect{P}$. 
The first constraints are the consistency conditions with the distributions $\vect{a}, \vect{b}$, 
and the second and third constraints are the consistency conditions with the compositions. 

\subsection{Reduction}\label{subsec:Reduction}

We move on to the dual problem of $\OT{\sd{D}}{\vect{a}}{\vect{b}}$.
\begin{defn}
    \label{def:dualcSD}
    The \emph{dual problem of $\OT{\sd{D}}{\vect{a}}{\vect{b}}$} is given by 
    \begin{equation}
    \label{align:constDualcSD}
        \begin{gathered}
            \sup\limits_{\vect{f}\in \R^{m},\, \vect{g}\in \R^{n}} \sum^{m}_{i=1} f_i\cdot a_i + \sum^{n}_{j=1} g_j\cdot b_j,\\
            \text{ s.t. }  f_i + g_j \leq C^{\sd{D}}_{ij},
        \end{gathered}
    \end{equation}
     where $m = m^{\oOT{A}}$, $n = \sum_{i\in [l_{H}]} n^{\oOT{B}_{Hi}}$ and the matrix $\vect{C}^{\sd{D}}$ is given by 
     \begin{align*}
        \vect{C}^\oOT{A}\seqcomp (\vect{C}^{\oOT{B}_{11}}\otimes \dots \otimes \vect{C}^{\oOT{B}_{1l_1}})\seqcomp \dots \seqcomp (\vect{C}^{\oOT{B}_{H1}}\otimes \dots \otimes \vect{C}^{\oOT{B}_{Hl_H}}).
     \end{align*}
\end{defn}
We note that compositions $\seqcomp$ and $\otimes$ for cost matrices can be naturally extended to matrices whose elements can be $\infty$.
Strictly speaking, however, the definition of the matrix $\vect{C}^{\sd{D}}$ itself is not yet well-defined,
because the order of the compositions is not defined.
For instance, it is unclear whether the formula $\vect{C}^{\oOT{C}_1}\seqcomp \vect{C}^{\oOT{C}_2}\seqcomp \vect{C}^{\oOT{C}_3}$ represents $\vect{C}^{\oOT{C}_1}\seqcomp (\vect{C}^{\oOT{C}_2}\seqcomp \vect{C}^{\oOT{C}_3})$ or $(\vect{C}^{\oOT{C}_1}\seqcomp \vect{C}^{\oOT{C}_2})\seqcomp \vect{C}^{\oOT{C}_3}$.
We address this issue by revealing the underlying categorical structure in cost matrices. 

\begin{prop}
    \label{prop:SSMC}
    The cost matrices whose elements are real-value or $\infty$ form 
    a symmetric strict monoidal category (SSMC) whose objects are natural numbers and homomorphisms are cost matrices. 
\end{prop}
We do not recall the definition of SSMCs,
and we instead show some properties that are useful for our development of string diagrams of OTs. 
We refer to~\cite {mac2013categories,selinger2011survey} as references, and some details are given in the technical appendix ($\S$B). 

Thanks to~\cref{prop:SSMC}, we now know that the following associative laws hold: 
\begin{align}
    \vect{C}^{\oOT{C}_1}\seqcomp (\vect{C}^{\oOT{C}_2}\seqcomp \vect{C}^{\oOT{C}_3}) &= (\vect{C}^{\oOT{C}_1}\seqcomp \vect{C}^{\oOT{C}_2})\seqcomp \vect{C}^{\oOT{C}_3},\\
    \vect{D}^{\oOT{D}_1}\otimes(\vect{D}^{\oOT{D}_2}\otimes \vect{D}^{\oOT{D}_3}) &= (\vect{D}^{\oOT{D}_1}\otimes \vect{D}^{\oOT{D}_2})\otimes \vect{D}^{\oOT{D}_3}.
\end{align}
We can thus omit the parentheses in the definition of $\vect{C}^{\sd{D}}$ in~\cref{def:dualcSD}. 
Importantly, all elements in $\vect{C}^{\sd{D}}$ are real-value due to the aligned structure of $\sd{D}$.

\begin{thm}[strong duality]
    \label{thm:strong_duality}
    The  dual problem defined in~\cref{def:dualcSD} is equivalent to $\OT{\sd{D}}{\vect{a}}{\vect{b}}$.
\end{thm}
We present the full proof in the technical appendix ($\S$A). 
The proof is obtained by the structural induction of a given aligned string diagram. 
The structure ensures the existence of feasible solutions.
The following is a direct consequence that leads to our novel algorithm.  

\begin{cor}
    \label{cor:correctRD}
    The $\OT{\sd{D}}{\vect{a}}{\vect{b}}$ of the aligned string diagram $\sd{D}$ is equivalent to the (monolithic) $\OT{\vect{C}^{\sd{D}}}{\vect{a}}{\vect{b}}$. 
\end{cor}
Similar to~\cref{prop:eq_seqOT}, the equivalence is w.r.t.~minimum transportation costs, 
not w.r.t.~optimal transportation plans. 

\section{Algorithm}
\label{sec:reduction}
\begin{algorithm}[t]
    \caption{Computation of an optimal hierarchical transportation plan on the aligned string diagram $\sd{D}$}
    \begin{algorithmic}[1]
    \State \label{line:compMat}Get the cost matrix $\vect{C}^{\sd{D}} \gets  \vect{C}^\oOT{A}\seqcomp (\vect{C}^{\oOT{B}_{11}}\otimes \dots \otimes \vect{C}^{\oOT{B}_{1l_1}})\seqcomp \dots \seqcomp (\vect{C}^{\oOT{B}_{k1}}\otimes \dots \otimes \vect{C}^{\oOT{B}_{kl_k}})$.\Comment{recursively compute by the compositions over cost matrices.}
    \State \label{line:monReduction}Compute an optimal transportation plan  $\vect{P}$ on $\OT{\vect{C}^{\sd{D}}}{\vect{a}}{\vect{b}}$.
    \State \label{line:outputCompOPT}Synthesize a hierarchical transportation plan $\vect{P}^{\oOT{A}}, (\vect{P}^{\oOT{B}_{ji}})_{j,i}$ on $\OT{\sd{D}}{\vect{a}}{\vect{b}}$ from $\vect{P}$. \Comment{call~\cref{alg:SyntHTP}}
    \State \Return $\big(\vect{P}^{\oOT{A}}, (\vect{P}^{\oOT{B}_{ji}})_{j,i}\big)$.
    \end{algorithmic}
    \label{alg:CompOPT}
 \end{algorithm}
Based on the reduction procedure in \cref{subsec:Reduction}, we develop an algorithm to compute an optimal hierarchical transportation plan of the given aligned string diagram $\sd{D}$ in \cref{alg:CompOPT}.
In~\cref{line:monReduction}, we can use existing solvers for OTs.
In~\cref{line:outputCompOPT}, we construct an optimal hierarchical transportation plan from the optimal transportation of $\OT{\vect{C}^{\sd{D}}}{\vect{a}}{\vect{b}}$, where $\vect{a}\in \Delta^{m}$ and $\vect{b}\in \Delta^{n}$.


\subsection{Synthesizing Hierarchical Transportation Plans}
\begin{algorithm}[t]
    \caption{Synthesizing Hierarchical Transportation Plan}
    \begin{algorithmic}[1]
    \State \label{line:initHTP}Initialize matrices $\vect{P}^{\oOT{A}}, (\vect{P}^{\oOT{B}_{ji}})_{j,i}$ as $\vect{0}$.
    \ForEach {$x$ in $[m]$ and $y$ in $[n]$}
    \State \, Get the cached shortest path $\pi$ from $x$ to $y$.
    \State \label{line:updateHTP} Add $P_{xy}$ for any elements of  $\vect{P}^{\oOT{A}}, (\vect{P}^{\oOT{B}_{ji}})_{j,i}$ that appear in $\pi$.
    \EndFor
    \State \Return $\big(\vect{P}^{\oOT{A}}, (\vect{P}^{\oOT{B}_{ji}})_{j,i}\big)$.
    \end{algorithmic}
    \label{alg:SyntHTP}
 \end{algorithm}
We now show the detail of the procedure of~\cref{line:outputCompOPT} in~\cref{alg:CompOPT}. 
In~\cref{alg:SyntHTP}, we outline this synthesis of the optimal hierarchical transportation plan.
In~\cref{line:compMat} in~\cref{alg:SyntHTP}, we initialize matrices $\vect{P}^{\oOT{A}}, (\vect{P}^{\oOT{B}_{ji}})_{j,i}$ as $\vect{0}$, which will be returned as an optimal hierarchical transportation plan.
For each pair of $(x, y)\in [m]\times [n]$ ($m$ and $n$ are the dimension of distributions $\vect{a}$ and $\vect{b}$, respectively), we update these matrices with (i) the optimal transportation plan $\vect{P}$ on $\OT{\vect{C}^{\sd{D}}}{\vect{a}}{\vect{b}}$ obtained from~\cref{line:monReduction} in~\cref{alg:CompOPT}; and (ii) one of the shortest paths along cost matrices from $x$ to $y$ on $\sd{D}$, which is efficiently cached when we compose cost matrices in~\cref{line:compMat}.
In~\cref{line:updateHTP} in~\cref{alg:SyntHTP}, we add $P_{xy}$ for any entries in $\vect{P}^{\oOT{A}}, (\vect{P}^{\oOT{B}_{ji}})_{j,i}$ that appears in the cashed shortest path $\pi$ from $x$ to $y$. 

\begin{thm}
The output of~\cref{alg:SyntHTP} is an optimal hierarchical transportation plan. 
\end{thm}
\begin{proofs}
By~\cref{thm:strong_duality}, it suffices to show that the hierarchical transportation cost by $\vect{P}^{\oOT{A}}, (\vect{P}^{\oOT{B}_{ji}})_{j,i}$ coincides with the optimal transportation cost by $\vect{P}$, which is easy.  
\end{proofs}

\subsection{Beyond Aligned String Diagram}
So far, we have focused on aligned string diagrams whose structures are $\oOT{A}\seqcomp (\oOT{B}_{11}\otimes \dots \otimes \oOT{B}_{1l_1})\seqcomp \dots \seqcomp (\oOT{B}_{H1}\otimes \dots \otimes \oOT{B}_{Hl_H})$. 
We relax the restriction on the structure, allowing $\oOT{A}\seqcomp \big((\oOT{B}\seqcomp \oOT{C}) \otimes \oOT{D}\big)$, for instance.  
Here, we define general string diagrams on OTs.
\begin{figure}
    \centering
    \begin{minipage}[b]{0.4\columnwidth}
    \begin{prooftree}
        \AxiomC{$\oOT{A} \defeq (m, n, \vect{C}^{\oOT{A}})$}
        \LeftLabel{$\axrule$}
        \UnaryInfC{$\oOT{A}\colon m\rightarrow n$}
    \end{prooftree}
    \end{minipage}
    \begin{minipage}[b]{0.4\columnwidth}
    \begin{prooftree}
       \AxiomC{$n\in \N$}
       \LeftLabel{$\idrule$}
       \UnaryInfC{$(n, n, \idcost{n})\colon n\rightarrow n$}
    \end{prooftree}
    \end{minipage}
    \begin{minipage}[b]{0.6\columnwidth}
    \begin{prooftree}
        \AxiomC{$\sd{D}_1\colon m\rightarrow l$}
        \AxiomC{$\sd{D}_2\colon l\rightarrow n$}
        \LeftLabel{$\seqrule$}
        \BinaryInfC{$\sd{D}_1\seqcomp \sd{D}_2\colon m\rightarrow n$}
    \end{prooftree}
    \end{minipage}
    \begin{minipage}[b]{0.6\columnwidth}
       \begin{prooftree}
           \AxiomC{$\sd{D}_1\colon m\rightarrow n$}
           \AxiomC{$\sd{D}_2\colon k\rightarrow l$}
           \LeftLabel{$\sumrule$}
           \BinaryInfC{$\sd{D}_1\otimes \sd{D}_2\colon m+k\rightarrow n+l$}
       \end{prooftree}
       \end{minipage}
    \caption{The typing rules for string diagrams.}
    \label{fig:TypeSystem}
 \end{figure}
A \emph{string diagram} $\sd{D}$ is a formula generated by the typing rules shown in~\cref{fig:TypeSystem}. 
The cost matrix $\vect{C}^{\sd{D}}$ of $\sd{D}$ is also inductively defined by (i) $\vect{C}^{\oOT{A}}$ if $\sd{D} = \oOT{A}$; 
(ii) $\idcost{n}$ if $\sd{D} = (n, n, \idcost{n})$;
(iii) $\vect{C}^{\oOT{A}}\seqcomp \vect{C}^{\oOT{B}}$ if $\sd{D}= \oOT{A}\seqcomp \oOT{B}$;
(iv) $\vect{C}^{\oOT{A}}\otimes \vect{C}^{\oOT{B}}$ if $\sd{D}= \oOT{A}\otimes \oOT{B}$.
Then, there is a simple reduction from a string diagram $\oOT{A}\seqcomp \sd{D}_1$ to an aligned string diagram $\oOT{A}\seqcomp \sd{D}_2$, which preserves the cost matrix. 
The string diagram $\sd{D}_2$ is \emph{sequential normal form}: 
$\sd{D}_2$ is given by $\sd{D}_2\defeq  (\oOT{B}_{11}\otimes \dots \otimes \oOT{B}_{1l_1})\seqcomp \dots \seqcomp (\oOT{B}_{H1}\otimes \dots \otimes \oOT{B}_{Hl_H})$. 
\begin{prop}
    \label{prop:SDToCSD}
    Given an oOT $\oOT{A}\colon m\rightarrow l$ and a string diagram $\sd{D}_1\colon l\rightarrow n$,
    there is an aligned string diagram $\oOT{A}\seqcomp \sd{D}_2\colon m\rightarrow n$ such that 
    (i) each oOT in $\sd{D}_2$ bijectively corresponds to an oOT in $\sd{D}_1$ except for identities $(k, k, \idcost{k})$;
    (ii) $\sd{D}_2\colon l\rightarrow n$ is sequential normal form; 
    (iii) the equation 
    $\vect{C}^{\oOT{A}\seqcomp \sd{D}_1} = \vect{C}^{\oOT{A}\seqcomp\sd{D}_2}$ holds.  
\end{prop}
The proof is based on the structural induction; see the technical appendix ($\S$B) for the details. 
 By~\cref{prop:SDToCSD} and its constructive proof, we can define the OT on string diagrams $\oOT{A}\seqcomp \sd{D}$, without assuming that $\oOT{A}\seqcomp \sd{D}$ is aligned. 

\section{Choice of Cost Matrices}
\label{sec:choiceCostMats}
In this section, we develop an extension of string diagrams of OTs with (non-deterministic) choices of cost matrices. For simplicity, we fix the finite set of choices as $C$; it is easy to extend that it allows to employ an individual set of choices $C^{\oOT{A}}$ for each component oOT $\oOT{A}$.  
We extend oOTs into oOTs with choices of cost matrices.
Formally, an oOT $\oOT{A}$ with choices of cost matrices is given by $\big(m, n, (\vect{C}^{\oOT{A}, c})_{c\in C}\big)$, where each choice $c$ induces the cost matrix $\vect{C}^{\oOT{A}, c}\in \Rnneg^{m\times n}$.
Now the planning problem must also consider an adversarial choice of cost matrices that tries to make the transportation cost large. 
\begin{defn}
\label{def:otsChoices}
Let $\vect{a} \in \Delta^{m}$ and $\vect{b} \in \Delta^{n}$ be two distributions, where $n\defeq \sum_{i\in [l_H]} n^{\oOT{B}_{Hi}}$. 
    The $\OT{\sd{D}}{\vect{a}}{\vect{b}}$ \emph{with the choice $C$} on an aligned string diagram $\sd{D}$ is given by 
    \begin{equation*}
        \begin{array}{cl}
        &\max\limits_{\vect{C}^{\oOT{A}}, (\vect{C}^{\oOT{B}_{ji}})_{j,i}}\min\limits_{\vect{P}^{\oOT{A}}, (\vect{P}^{\oOT{B}_{ji}})_{j,i}} 
        \langle \vect{C}^{\oOT{A}}, \vect{P}^{\oOT{A}} \rangle +  \sum\limits_{j,i}\langle \vect{C}^{\oOT{B}_{ji}}, \vect{P}^{\oOT{B}_{ji}}\rangle,
        \end{array}
    \end{equation*} 
    where (i) $\vect{P}^{\oOT{A}}\in \Rnneg^{m^{\oOT{A}}\times n^{\oOT{A}}}$ and $\vect{P}^{\oOT{B}_{ji}}\in \Rnneg^{m^{\oOT{B}_{ji}}\times n^{\oOT{B}_{ji}}}$ are subject to the same constraints defined in~\cref{def:SDOT}; and (ii) $\vect{C}^{\oOT{A}}$ and $(\vect{C}^{\oOT{B}_{ji}})_{j,i}$ should be chosen from the finite sets $\{\vect{C}^{\oOT{A}, c}\mid c\in C\}$ and $\{\vect{C}^{\oOT{B}_{ji}, c}\mid c\in C\}$ for each $i, j$, respectively. We note that the choice $c$ can be done independently for each component, that is, components can choose different choices to achieve the optimal transportation cost. 
\end{defn}
Naively, we can solve the problem by enumerating all choices of cost matrices, the number of which is exponential in the number of compositions.
In practice, this easily becomes infeasible as the number of compositions increases.
Instead, we relax the problem by allowing convex combinations of cost matrices to be chosen, which is crucial for our reduction to LP.
\begin{defn}
\label{def:relaxOtsChoices}
The \emph{relaxation problem of~\cref{def:otsChoices}} is defined by 
    \begin{equation*}
        \begin{array}{cl}
        &\max\limits_{\vect{D}^{\oOT{A}}, (\vect{D}^{\oOT{B}_{ji}})_{j,i}}\min\limits_{\vect{P}^{\oOT{A}}, (\vect{P}^{\oOT{B}_{ji}})_{j,i}} 
        \langle \vect{D}^{\oOT{A}}, \vect{P}^{\oOT{A}} \rangle +  \sum\limits_{j,i}\langle \vect{D}^{\oOT{B}_{ji}}, \vect{P}^{\oOT{B}_{ji}}\rangle,
        \end{array}
    \end{equation*} 
    where $\vect{P}^{\oOT{A}}\in \Rnneg^{m^{\oOT{A}}\times n^{\oOT{A}}}$ and $\vect{P}^{\oOT{B}_{ji}}\in \Rnneg^{m^{\oOT{B}_{ji}}\times n^{\oOT{B}_{ji}}}$, subject to (i) the same constraints defined in~\cref{def:otsChoices}; and (ii) $\vect{D}^{\oOT{A}}$ and $(\vect{D}^{\oOT{B}_{ji}})_{j,i}$ should be chosen from the convex closures of $\{\vect{C}^{\oOT{A}, c}\mid c\in C\}$ and $\{\vect{C}^{\oOT{B}_{ji}, c}\mid c\in C\}$, respectively. 
\end{defn}

The optimal value of~\cref{def:relaxOtsChoices} can be strictly greater than the one of~\cref{def:otsChoices}: see~\cref{sec:relaxation} for the example. 
Our reduction to LP relies on the following mini-max principle, which allows us to exchange the order of $\min$ and $\max$. 
\begin{lemma}[mini-max principle]
The following problem is equivalent to one defined in~\cref{def:relaxOtsChoices}: 
\begin{equation*}
        \begin{array}{cl}
        &\min\limits_{\vect{P}^{\oOT{A}}, (\vect{P}^{\oOT{B}_{ji}})_{j,i}} \max\limits_{\vect{D}^{\oOT{A}}, (\vect{D}^{\oOT{B}_{ji}})_{j,i}} \langle \vect{D}^{\oOT{A}}, \vect{P}^{\oOT{A}} \rangle +  \sum\limits_{j,i}\langle \vect{D}^{\oOT{B}_{ji}}, \vect{P}^{\oOT{B}_{ji}}\rangle
        \end{array}
    \end{equation*} 
    under the constraints that are defined in~\cref{def:otsChoices}.
\end{lemma}
\begin{proofs}
Since the domains are compact and convex, the proof follows from the existence of the pure Nash equilibrium in zero-sum games with concave-convex objectives (see e.g.~\cite{mazalov2014mathematical}). 
\end{proofs}

Now, we can see that the problem is a \emph{piecewise-linear minimization}, which allows us to use a reduction to LP.  

\begin{defn}
\label{def:LPofchoice}
    Let $\vect{a} \in \Delta^{m}$ and $\vect{b} \in \Delta^{n}$ be two distributions, where $n\defeq \sum_{i\in [l_H]} n^{\oOT{B}_{Hi}}$, and let $\sd{D}$ be an aligned string diagram. 
    We introduce an LP by 
    \begin{equation*}
        \begin{array}{cl}
        &\min\limits_{\vect{P}^{\oOT{A}}, (\vect{P}^{\oOT{B}_{ji}})_{j,i},t^{\oOT{A}}, (t^{\oOT{B}_{ji}})_{j,i}}  t^{\oOT{A}} +  \sum\limits_{j,i}t^{\oOT{B}_{ji}},
        \end{array}
    \end{equation*} 
    subject to the same constraints defined in~\cref{def:otsChoices}, and the following additional constraints: 
    \begin{align*}
       \langle \vect{C}^{\oOT{A},c}, \vect{P}^{\oOT{A}} \rangle  &\leq t^{\oOT{A}} &&\forall c\in C,\\
       \langle \vect{C}^{\oOT{B}_{ji},c}, \vect{P}^{\oOT{B}_{ji}} \rangle  &\leq t^{\oOT{B}_{ji}} &&\forall c\in C,\forall j\in[l_i]\text{ and }\forall i\in [H].
    \end{align*}
\end{defn}

\begin{thm}
    \label{thm:strong_duality_choice}
    The problem defined in~\cref{def:relaxOtsChoices} is equivalent to the LP defined in~\cref{def:LPofchoice}.
\end{thm}
The proof follows easily from the standard reduction of piecewise-linear to linear minimization.
As a consequence, the problem is solvable in polynomial time (by restricting all coefficients to be rational).

\section{Experiments}
\label{sec:experiments}
We demonstrate the efficiency of~\cref{alg:CompOPT} for string diagram of OTs by implementing a prototype in Python. 
We implemented two algorithms: 
(i) the \emph{composed LP approach} ($\composedLP$), which directly solves the LP shown in~\cref{def:SDOT}; and
(ii) the \emph{monolithic approach} ($\monLP$), which is based on~\cref{alg:CompOPT}.
\ifsubmission
\else
\footnote{Code and datasets are available at \url{https://github.com/Kazuuuuuki/SolverForStringDiagramOfOTs}}
\fi
The questions that we pose are 
{\bf Q1}: do our proposed monolithic approach ($\monLP$) performs better than the naive LP approach ($\composedLP$),
{\bf Q2}: what is the bottlenecks of $\monLP$, 
and  {\bf Q3}:  how does the complexity of compositional structures affect their performances?

\paragraph{Setup.} 
We ran experiments on a MacBook Pro Apple M3 chip 12-core CPU with a 16GB memory limit. We use the \emph{PuLP} toolkit~\cite{mitchell2011pulp} and the \emph{CBC} solver~\cite{CBCsolver} for LPs.

\paragraph{Benchmarks.} 
As a first set of benchmarks for string diagrams of OTs, we build four benchmarks $\texttt{BRooms}$, $\texttt{URooms}$, $\texttt{BChains}$, $\texttt{UChains}$ that reflect common hierarchical structures found in real-world problems, particularly in hierarchical planning and network protocols~\cite{cav/WatanabeEAH23,tacas/WatanabeVHRJ24,cav/JungesS22,JothimuruganBBA21}.
Our benchmarks are inspired by the standard benchmarks for string diagrams of MDPs~\cite{tacas/WatanabeVHRJ24}. 
In fact, most of the benchmarks in~\cite{tacas/WatanabeVHRJ24} take the form of aligned string diagrams. 
For instance, the benchmark \texttt{Biroom} in~\cite{tacas/WatanabeVHRJ24} takes the form $\oOT{A}\seqcomp (\oOT{B}_{11}\otimes \dots \otimes \oOT{B}_{1l_1})\seqcomp \dots \seqcomp (\oOT{B}_{k1}\otimes \dots \otimes \oOT{B}_{kl_k})\seqcomp \oOT{C}$---
we create the benchmarks \texttt{BRooms} and \texttt{URooms} by resembling its compositional structure. 
Here, each string diagram in the two benchmarks \texttt{BRooms} and \texttt{URooms} represents buildings consisting of many rooms connected by the two compositions.
Each room is modeled by an oOT whose cost matrix denotes the time required to get from the entrances to the exits.  
Similarly, the string diagrams of the two benchmarks \texttt{BChains} and \texttt{UChains} are very close to the string diagrams of MDPs in the benchmark \texttt{ChainsBig}~\cite{tacas/WatanabeVHRJ24}. See~\cref{appendix:benchmarkDetails} for more details. 

\begin{table*}[t]
    \centering
     \caption{Performance for different algorithms. See \textbf{Discussion} for explanations.}
{    \small
    \begin{NiceTabular}{rr  rrrrr  rr}
\toprule
& & 
\multicolumn{5}{c}{$\monLP$} & 
\multicolumn{2}{c}{$\composedLP$}\\
\cmidrule(lr){3-7}\cmidrule(lr){8-9}

\multicolumn{1}{c}{$\sd{D}$} &
\multicolumn{1}{c}{$\#\oOT{A}$} &

\multicolumn{1}{c}{$t_{\vect{C}}$} &
\multicolumn{1}{c}{$t_{\mathrm{LP}}$} &
\multicolumn{1}{c}{$t_{\mathrm{Syn}}$} &
\multicolumn{1}{c}{$t$} &
\multicolumn{1}{c}{$E$} &

\multicolumn{1}{c}{$t$} &
\multicolumn{1}{c}{$E$} \\
\midrule

\multicolumn{1}{c}{\texttt{BRoom1}} &
\multicolumn{1}{c}{200} &

\multicolumn{1}{c}{3.1}  &
\multicolumn{1}{c}{0.1}  &
\multicolumn{1}{c}{0.1}  &
\multicolumn{1}{c}{3.3}  &
\multicolumn{1}{c}{1e-16}  &

\multicolumn{1}{c}{17.6}  &
\multicolumn{1}{c}{1e-15}  \\

\multicolumn{1}{c}{\texttt{BRoom2}} &
\multicolumn{1}{c}{210} &

\multicolumn{1}{c}{32.6}  &
\multicolumn{1}{c}{0.1}  &
\multicolumn{1}{c}{0.1}  &
\multicolumn{1}{c}{32.8}  &
\multicolumn{1}{c}{0.0}  &

\multicolumn{1}{c}{144.8}  &
\multicolumn{1}{c}{1e-16}  \\

\multicolumn{1}{c}{\texttt{URoom1}} &
\multicolumn{1}{c}{400} &

\multicolumn{1}{c}{0.4}  &
\multicolumn{1}{c}{0.1}  &
\multicolumn{1}{c}{0.1}  &
\multicolumn{1}{c}{0.6}  &
\multicolumn{1}{c}{2e-16}  &

\multicolumn{1}{c}{42.0}  &
\multicolumn{1}{c}{3e-16}  \\

\multicolumn{1}{c}{\texttt{URoom2}} &
\multicolumn{1}{c}{600} &

\multicolumn{1}{c}{0.8}  &
\multicolumn{1}{c}{0.1}  &
\multicolumn{1}{c}{0.1}  &
\multicolumn{1}{c}{\textbf{1.0}}  &
\multicolumn{1}{c}{1e-16}  &

\multicolumn{1}{c}{\textbf{95.4}}  &
\multicolumn{1}{c}{3e-16}  \\

\multicolumn{1}{c}{\texttt{BChain1}} &
\multicolumn{1}{c}{210} &

\multicolumn{1}{c}{11.3}  &
\multicolumn{1}{c}{0.1}  &
\multicolumn{1}{c}{0.2}  &
\multicolumn{1}{c}{11.6}  &
\multicolumn{1}{c}{2e-16}  &

\multicolumn{1}{c}{103.2}  &
\multicolumn{1}{c}{9e-16}  \\

\multicolumn{1}{c}{\texttt{BChain2}} &
\multicolumn{1}{c}{400} &

\multicolumn{1}{c}{25.0}  &
\multicolumn{1}{c}{0.1}  &
\multicolumn{1}{c}{0.3}  &
\multicolumn{1}{c}{25.4}  &
\multicolumn{1}{c}{1e-16}  &

\multicolumn{1}{c}{311.0}  &
\multicolumn{1}{c}{1e-16}  \\

\multicolumn{1}{c}{\texttt{UChain1}} &
\multicolumn{1}{c}{399} &

\multicolumn{1}{c}{0.5}  &
\multicolumn{1}{c}{0.1}  &
\multicolumn{1}{c}{0.1}  &
\multicolumn{1}{c}{\textbf{0.7}}  &
\multicolumn{1}{c}{0.0}  &

\multicolumn{1}{c}{\textbf{86.9}}  &
\multicolumn{1}{c}{4e-16}  \\

\multicolumn{1}{c}{\texttt{UChain2}} &
\multicolumn{1}{c}{799} &

\multicolumn{1}{c}{1.3}  &
\multicolumn{1}{c}{0.1}  &
\multicolumn{1}{c}{0.1}  &
\multicolumn{1}{c}{\textbf{1.5}}  &
\multicolumn{1}{c}{6e-16}  &

\multicolumn{1}{c}{\textbf{266.8}}  &
\multicolumn{1}{c}{3e-15}  \\
\bottomrule
    \end{NiceTabular}
    }
   \label{tab:runtimesetc}
\end{table*}

\begin{figure*}
    \center
    \vspace{-1em}
    \includegraphics[width=0.48\linewidth]{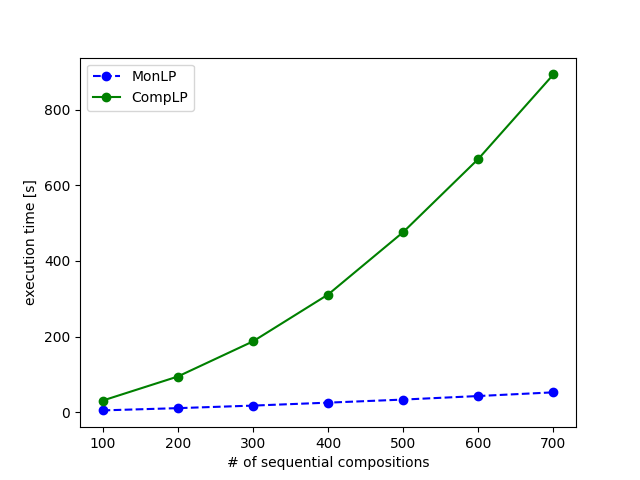}
    \includegraphics[width=0.48\linewidth]{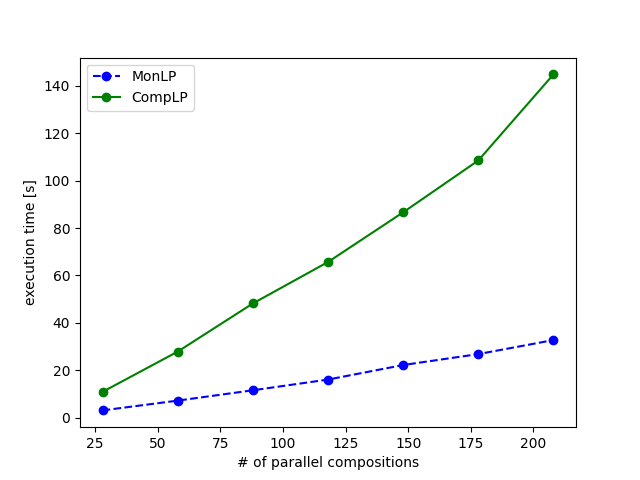}
    \caption{Influence of the complexity of the algebraic structures. See \textbf{Discussion} for explanations.}
    \label{fig:infAlgSt}
\end{figure*}
\paragraph{Discussion.} 
\cref{tab:runtimesetc} and~\cref{fig:infAlgSt} summarize the results of the experiments. 

{\bf Q1.} 
We evaluate the performance of $\monLP$, and $\composedLP$ with the benchmarks by 
computing a near-optimal transportation plan.  
\cref{tab:runtimesetc} gives the details of the result. 
The columns give the string diagram $\sd{D}$, the number of oOTs $\oOT{A}$ in $\sd{D}$, the execution time $t_{\vect{C}}$ for composing cost matrices (\cref{line:compMat} in~\cref{alg:CompOPT}),
the execution times $t_{\mathrm{LP}}$ for the (monolithic) LP (\cref{line:monReduction} in~\cref{alg:CompOPT}), the execution time $t_{\mathrm{Syn}}$ for the synthesis (\cref{alg:SyntHTP}),  the execution time $t$ for the entire algorithms, and the relative error $E$ of minimum transportation costs w.r.t. exact solutions. 
The unit of time is seconds.

\cref{tab:runtimesetc} shows that $\monLP$ beats $\composedLP$ in terms of the speed and the precision for every benchmark. 
In particular, $\monLP$ outperforms $\composedLP$ by orders of magnitude for \texttt{URoom2}, \texttt{UChain1}, and \texttt{UChain2}.

{\bf Q2.} \cref{tab:runtimesetc} clearly indicates that the bottleneck is the computation for cost matrices (\cref{line:compMat} in~\cref{alg:CompOPT}).
For instance, the result of \texttt{BRoom2} shows that the computation of cost matrices takes about $\SI{32}{sec}$ while solving the monolithic OT and synthesis take less than $\SI{1}{sec}$.  
The results imply that efficient representations and computations for matrices over the min-tropical semiring  have a chance to improve performance: this is indeed an interesting future direction. 

{\bf Q3.} We evaluate how the number of compositions affects the performance of the algorithms with the benchmarks $\texttt{BChain}$ and $\texttt{BRooms}$: 
the left figure in \cref{fig:infAlgSt} shows the performance change by increasing the number of the sequential compositions $\seqcomp$ in $\texttt{BChains}$, and the right figure shows that of the parallel compositions $\otimes$ in $\texttt{BRooms}$. 
We can see that $\monLP$ has a clear performance advantage over $\composedLP$ for more complicated structures that are composed with these two algebraic operations.

\section{Related Work}\label{sec:relatedWork}
\paragraph{Optimal Transport with Structures.}
As OT finds increasing applications in machine learning, several challenges have emerged that classical discrete OT cannot address.
One challenge arises when the (discrete) distributions are high-dimensional, and another occurs when there are more than three marginal distributions involved.
The latter is known as the multi-marginal OT (MMOT), which is computationally challenging due to the exponential complexity w.r.t. the number of marginal distributions.
See \cite{Pass2015,jmlr/LinHCJ22}.

To tackle these issues, efficient algorithms have been proposed by assuming certain structures in the distributions or cost matrices, such as hierarchy or symmetry of distributions~\cite{ml/HamriBF22,scalespace/SchmitzerS13,Takeda_Akagi_Marumo_Niwa_2024}, and submodularity of the cost matrix~\cite{aistats/Alvarez-MelisJJ18}.
More recently, algorithms exploiting the hierarchical structure of distributions on graphs have been developed for MMOT in \cite{Zeng_Du_Zhang_Xia_Liu_Tong_2024}.

As a comparison, we propose a planning-motivated variant of OT and discover the new algebraic structure.
%
%
\paragraph{String Diagram.}
String diagrams (e.g.~\cite{mac2013categories,joyal1991geometry,selinger2011survey}) have been widely studied as a graphical expression of monoidal categories. 
There are many applications of string diagrams, such as recurrent neural networks~\cite{SprungerK19}, gradient-based learning~\cite{CruttwellGGWZ22}, quantum computation~\cite{CK2017,AbramskyC04}, and disentanglement in machine learning~\cite{ZhangS23}. 
\paragraph{Hierarchical Reinforcement Learning.}
The assumption of a hierarchical structure in models is popular in reinforcement learning. 
Hierarchical reinforcement learning~\cite{deds/BartoM03,csur/PateriaSTQ21} performs high- (or macro-) and low-level (or micro-level) planning on such hierarchical models. 

\section{Conclusion}
We introduce string diagrams as a hierarchical framework of OTs, and provide a new algorithm that exploits the algebraic structure over cost matrices. 
In experiments, our algorithm outperforms naive LP in all benchmarks.

One of the future works is to study the complexity class of the string diagrams of OTs with choices (without the relaxation).
Another interesting direction is to extend the framework of implicit Sinkhorn differentiation framework developed in, e.g., \cite{cvpr/EisenbergerTLBC22} by exploiting the algebraic structure of the cost matrices.


\bibliographystyle{named}
\bibliography{named}

\iffull

\onecolumn
\appendix

\section{Proof of Strong Duality}
\label{appendix:proofSD}
Let $\vect{I}_{m}\in\R^{m\times m}$ denote the identity matrix, i.e., $\vect{I}_m\defeq\begin{pNiceMatrix}[nullify-dots]
             1 & & \Block[l]{1-1}<\huge>{{0}} \\
             & \Ddots  &  \\
             \Block[r]{1-1}<\huge>{{0}} &  & 1\\
        \end{pNiceMatrix}$.
Note that $\vect{I}_m\neq\idcost{m}=\begin{pNiceMatrix}[nullify-dots]
             0 & & \Block[l]{1-1}<\huge>{{\infty}} \\
             & \Ddots  &  \\
             \Block[r]{1-1}<\huge>{{\infty}} &  & 0\\
        \end{pNiceMatrix}$.

We first recall that strong duality in linear programming, see, e.g., \cite[Theorem 4.4]{bertsimas1997introduction} for a proof. 
\begin{thm}[strong duality in linear programming]
    \label{thm:strong_duality_LP}
    Let $\vect{A}$ be a real-matrix, $\vect{b}$ and $\vect{c}$ be extended-real-vectors with dimensions in the rows and columns of $\vect{A}$, respectively.
    Givem a \emph{primal} problem given as
    \begin{mini*}|l|
        {\substack{\vect{x}}} 
        {\left\langle \vect{c},\vect{x}\right\rangle} 
        {\label{opt:standardPrimal}} 
        {} 
        \addConstraint{\vect{A}\vect{x}}{=\vect{b}} 
        \addConstraint{\vect{x}}{\geq\vect{0},} 
    \end{mini*}
    its \emph{dual} problem is set to be another problem as
    \begin{maxi*}|l|
        {\substack{\vect{p}}} 
        {\left\langle \vect{p},\vect{b}\right\rangle} 
        {\label{opt:standardDual}} 
        {} 
        \addConstraint{\vect{A}^\top\vect{p}}{\leq\vect{c}.} 
    \end{maxi*}
    If there exists a feasible minimizer $\vect{x}^\ast$ in the primal problem, then there also exists a feasible maximizer $\vect{p}^\ast$ in the dual problem.
    Moreover, the minimum value of the primal equals the maximal value of the dual, i.e., $\langle \vect{c},\vect{x}^\ast\rangle=\langle\vect{p}^\ast,\vect{b}\rangle$.
\end{thm}

Based on \cref{thm:strong_duality_LP}, we prove \cref{prop:strong_duality_Seq,prop:strong_duality_Par}.
In the following proof, we will denote by $\vectop\vect{Z}$ the vectorization of a matrix $\vect{Z}$, i.e., for $\vect{Z}\in\R^{m\times n}$,
\[
   \vectop \vect{Z}\defeq\begin{pmatrix}
        Z_{1,1},\dots,Z_{m,1},\dots,Z_{1,n},\dots,Z_{m,n}
    \end{pmatrix}
    ^\top\in\R^{mn}.
\]
It is well known that, for any matrices $\vect{X},\vect{Y},\vect{Z}$ such that the product $\vect{XYZ}$ can be defined,
\begin{equation}
    \label{eq:Roth_formula}
    \vectop(\vect{XYZ})=\left(\vect{Z}^\top\Kron\vect{X}\right)\vectop{\vect{Y}},
\end{equation}
where $\Kron$ is the Kronecker product.
\begin{proof}[Proof of \cref{prop:strong_duality_Seq}]
    Note that $\vect{P}^\oOT{A}=l^{-1}\vect{a}\Kron\vect{1}_l$ and $\vect{P}^\oOT{B}=l^{-1}\vect{1}_l\Kron\vect{b}$ belong to the feasible region of \eqref{align:seqOT}.
    Thus, the region is nonempty and compact, and there exists a minimizer from the Weierstrass theorem~\cite[Theorem 2.12]{Beck2017}.
    
    Using the formula \eqref{eq:Roth_formula}, we can rewrite \eqref{align:seqOT} as
    \begin{mini*}|l|
        {\substack{\vectop{\vect{P}^{\oOT{A}}},\vectop{\vect{P}^{\oOT{B}}}}} 
        {
        \left\langle
        \begin{pmatrix}
            \vectop\vect{C}^{\oOT{A}}\\
            \vectop\vect{C}^{\oOT{B}}
        \end{pmatrix}
        ,
        \begin{pmatrix}
            \vectop\vect{P}^\oOT{A}\\
            \vectop\vect{P}^\oOT{B}
        \end{pmatrix}
        \right\rangle
        } 
        {\label{opt:standardPrimal}} 
        {} 
        \addConstraint{
        \begin{pmatrix}
            \vect{1}_l^\top\Kron\vect{I}_{m}, & 0\\ 
            0, & \vect{I}_{n}\Kron\vect{1}_l^\top\\
            \vect{I}_{l}\Kron\vect{1}_m^\top, & -\vect{1}_n^\top\Kron\vect{I}_{l}
        \end{pmatrix}
        \begin{pmatrix}
            \vectop\vect{P}^\oOT{A}\\
            \vectop\vect{P}^\oOT{B}
        \end{pmatrix}
        }{=
        \begin{pmatrix}
            \vect{a}\\
            \vect{b}\\
            \vect{0}
        \end{pmatrix}
        } 
        \addConstraint{\begin{pmatrix}
            \vectop\vect{P}^\oOT{A}\\
            \vectop\vect{P}^\oOT{B}
        \end{pmatrix}}{\geq\vect{0}.} 
    \end{mini*}
    Therefore, \cref{thm:strong_duality_LP} yields the dual problem as
    \begin{maxi!}|l|
        {\substack{\vect{f},\vect{g},\vect{h}}} 
        {\left\langle
        \begin{pmatrix}
            \vect{f}\\
            \vect{g}\\
            \vect{h}
        \end{pmatrix}
        ,
        \begin{pmatrix}
            \vect{a}\\
            \vect{b}\\
            \vect{0}
        \end{pmatrix}
        \right\rangle} 
        {\label{opt:SeqDual}} 
        {} 
        \addConstraint{
        \begin{pmatrix}
            \vect{1}_l\Kron\vect{I}_{m} & 0  & \vect{I}_{l}\Kron\vect{1}_m\\ 
            0 & \vect{I}_{n}\Kron\vect{1}_l & -\vect{1}_n\Kron\vect{I}_{l}
        \end{pmatrix}
        \begin{pmatrix}
            \vect{f}\\
            \vect{g}\\
            \vect{h}
        \end{pmatrix}
        }{\leq
        \begin{pmatrix}
            \vectop\vect{C}^{\oOT{A}}\\
            \vectop\vect{C}^{\oOT{B}}
        \end{pmatrix}
        .
        \label{eq:matrix_form_fgh}
        } 
        \end{maxi!}
        Note that the objective function \eqref{opt:SeqDual} does not depend on $\vect{h}$.
        Using \eqref{eq:Roth_formula} again, the constraint \eqref{eq:matrix_form_fgh} can be rewritten as 
        \begin{align*}
            \vect{f1}_l^\top+\vect{1}_m\vect{h}^\top&\leq\vect{C}^{\oOT{A}},\\
            \vect{1}_l\vect{g}^\top - \vect{h}\vect{1}_n^\top&\leq\vect{C}^{\oOT{B}},
        \end{align*}
        which is written componentwisely as
        \[
           g_j-C_{lj}^{\oOT{B}} \leq h_l\leq C^{\oOT{A}}_{il}-f_i,
        \]
        for any $i\in[m]$, $k\in[l]$, and $m\in[n]$.
        By applying Fourier--Motzkin elimination to project out $\vect{h}$, we get \cref{align:constDualSeq}.
\end{proof}

\begin{proof}[Proof of \cref{prop:strong_duality_Par}]
    It is immediate that the minimum of \eqref{prob:ParOT} can be expressed as the sum of the minimums of the following two problems:
    \begin{equation}
    \begin{array}{cl}
            \min\limits_{\substack{\vect{P}^{\oOT{A}}\in \Rnneg^{m\times n}}} & \langle \vect{C}^{\oOT{A}}, \vect{P}^{\oOT{A}} \rangle  \\
        \text{s.t.} & \begin{aligned}
                             &\vect{P}^{\oOT{A}}\vect{1}_{n} = (a_i)_{i=1}^m, (\vect{P}^{\oOT{A}})^{\top}\vect{1}_{m} = (b_j)_{j=1}^n,
                        \end{aligned}
       
    \end{array}
    \label{prob:ParOT_C}
    \end{equation}
    \begin{equation}
    \begin{array}{cl}
            \min\limits_{\substack{
            \vect{P}^{\oOT{B}}\in \Rnneg^{k\times l}}} & \langle \vect{C}^{\oOT{B}}, \vect{P}^{\oOT{B}}\rangle \\
        \text{s.t.} & \begin{aligned}
                            &\vect{P}^{\oOT{B}}\vect{1}_{l} = (a_i)_{i=m+1}^{m+k}, (\vect{P}^{\oOT{B}})^{\top}\vect{1}_{k} = (b_j)_{j=n+1}^{n+l}.
                        \end{aligned}
       
    \end{array}
    \label{prob:ParOT_D}
    \end{equation}
    By the assumption $\sum^{m}_{i=1} a_i = \sum^{n}_{j=1}  b_j$, \eqref{prob:ParOT_C} and \eqref{prob:ParOT_D} can be viewed as optimal transport problems, respectively.
    Therefore, the duality of optimal transport (see, e.g., \cite[Proposition 2.4]{ftml/PeyreC19}) can be applied to both problems \eqref{prob:ParOT_C} and \eqref{prob:ParOT_D}.
\end{proof}

We finally provide a proof of \cref{thm:strong_duality}.

\begin{proof}[Proof of \cref{thm:strong_duality}]
Recall that the string diagram $\sd{D}$ is given by \cref{def:aligned_string_diagram}.
We prove the duality by induction on $H$.
\paragraph{(I) Base step: $\boldsymbol{H=1}$.} Given a positive integer $l$ and an aligned string diagram $\sd{D}_1\defeq\oOT{A}\seqcomp (\oOT{B}_{1}\otimes \dots \otimes \oOT{B}_{l})$, let us consider the problem
    \begin{mini}|l|
        {\substack{\vect{P}^{\oOT{A}}, (\vect{P}^{\oOT{B}_{i}})_{i=1}^l}} 
        {\langle \vect{C}^{\oOT{A}}, \vect{P}^{\oOT{A}} \rangle +  \sum_{i=1}^l\langle \vect{C}^{\oOT{B}_{i}}, \vect{P}^{\oOT{B}_{i}}\rangle} 
        {\label{opt:basePrimal}} 
        {} 
        \addConstraint{\vect{P}^{\oOT{A}}\vect{1}_{n^{\oOT{A}}} }{= \vect{a}} 
        \addConstraint{(\vect{P}^{\oOT{A}})^{\top}\vect{1}_{m^{\oOT{A}}}}{= (\vect{P}^{\oOT{B}})\vect{1}_{n^{\oOT{B}}}} 
        \addConstraint{(\vect{P}^{\oOT{B}})^\top\vect{1}_{m^{\oOT{B}}}}{=\vect{b},}
    \end{mini}
where
    \[
        \vect{P}^{\oOT{B}}\defeq 
        \begin{pNiceMatrix}[nullify-dots]
             \vect{P}^{\oOT{B}_{1}} & & \Block[c]{1-1}<\Huge>{{0}} \\
             & \Ddots  &  \\
             \Block[c]{1-1}<\Huge>{{0}} &  &   \vect{P}^{\oOT{B}_{l}}\\
        \end{pNiceMatrix}
         \in\Rnneg^{m^{\oOT{B}}\times n^{\oOT{B}}},
    \]
    and $m^{\oOT{B}} \defeq \sum_{i\in [l]} m^{\oOT{B}_{i}}$ and $n^{\oOT{B}} \defeq \sum_{i\in [l]} n^{\oOT{B}_{i}}$.
    It is obvious that 
    \[
    \begin{aligned}
        \vect{P}^\oOT{A}&=\vect{a}\Kron
        \begin{pmatrix}
            \left(b^{\oOT{B}_1}\middle/m^{\oOT{B}_1}\right)\vect{1}_{m^{\oOT{B}_1}}\\
            \vdots\\
            \left(b^{\oOT{B}_l}\middle/m^{\oOT{B}_l}\right)\vect{1}_{m^{\oOT{B}_l}}
        \end{pmatrix}, \\
        \vect{P}^{\oOT{B}_i}&=
        \begin{pNiceMatrix}[nullify-dots]
             b_{1+\sum_{k=1}^{i-1}n^{\oOT{B}_k}} & & \Block[c]{1-1}<\Huge>{{0}} \\
             & \Ddots  &  \\
             \Block[c]{1-1}<\Huge>{{0}} &  &   b_{n^{\oOT{B}_i}+\sum_{k=1}^{i-1}n^{\oOT{B}_k}}\\
        \end{pNiceMatrix}
    \end{aligned}
    \]
    is a feasible solution where
    \[
        {b}^{\oOT{B}_i}\defeq\sum_{j=1+\sum _{k=1}^{i-1}n^{\oOT{B}_k}}^{\sum _{k=1}^{i}n^{\oOT{B}_k}}b_j,
    \]
    for $i\in[l]$.
    By \cref{prop:ParOT_CotimesD}, the objective function of \eqref{opt:basePrimal} can be rewritten as
    \begin{equation}
        \label{eq:extended_objective}
        \langle \vect{C}^{\oOT{A}}, \vect{P}^{\oOT{A}} \rangle +  \langle \vect{C}^{\oOT{B}_{1}\otimes \dots \otimes \oOT{B}_{l}}, \vect{P}^{\oOT{B}}\rangle.
    \end{equation}
    Note that the off-block-diagonal elements of $\vect{C}^{\oOT{B}_{1}\otimes \dots \otimes \oOT{B}_{l}}$ are $\infty$, but there exists a feasible solution which belongs to the effective domain of \eqref{eq:extended_objective}.
    thus problem \eqref{opt:basePrimal} is equivalent to $\SeqOT{\oOT{A}}{\oOT{B}_{1}\otimes \dots \otimes \oOT{B}_{l}}{\vect{a}}{\vect{b}}$.
    In a similar manner to the proof of \cref{prop:strong_duality_Seq}, we can show that the dual of \eqref{opt:basePrimal} is given by the dual of $\OT{\vect{C}^{\sd{D}_1}}{\vect{a}}{\vect{b}}$, which is equivalent to the dual of $\OT{\sd{D}_1}{\vect{a}}{\vect{b}}$ defined by \cref{def:dualcSD}.

\paragraph{(II) Inductive step:}

Suppose that there exists a positive integer $H$ such that, for any aligned string diagrams with the form of \eqref{eq:aligned_string_diagram}.
Let us consider another diagram $\sd{D}_{H+1}$ with the form of 
\[
    \oOT{A}\seqcomp (\oOT{B}_{11}\otimes \dots \otimes \oOT{B}_{1l_1})\seqcomp \dots \seqcomp (\oOT{B}_{(H+1)1}\otimes \dots \otimes \oOT{B}_{(H+1)l_{H+1}}),
\]
and the corresponding problem $\OT{\sd{D}_{H+1}}{\vect{a}}{\vect{b}}$
    \begin{mini}|l|
        {\substack{\vect{P}^{\oOT{A}},((\vect{P}^{\oOT{B}_{ji}})_{i=1}^{l_j})_{j=1}^{H+1}}} 
        {\langle \vect{C}^{\oOT{A}}, \vect{P}^{\oOT{A}} \rangle +  \sum\limits_{j,i}\langle \vect{C}^{\oOT{B}_{ji}}, \vect{P}^{\oOT{B}_{ji}}\rangle} 
        {\label{opt:basePrimal_D_H+1}} 
        {} 
        \addConstraint{\vect{P}^{\oOT{A}}\vect{1}_{n^{\oOT{A}}} }{= \vect{a}} 
        \addConstraint{(\vect{P}^{\oOT{B}_{j}})^{\top}\vect{1}_{m^{\oOT{B}_{j}}}}{= (\vect{P}^{\oOT{B}_{j+1}})\vect{1}_{n^{\oOT{B}_{j+1}}}}{(j\in[H])}
        \addConstraint{(\vect{P}^{\oOT{A}})^{\top}\vect{1}_{m^{\oOT{A}}}}{= (\vect{P}^{\oOT{B}_1})\vect{1}_{n^{\oOT{B}_1}}} 
        \addConstraint{(\vect{P}^{\oOT{B}_{H+1}})^\top\vect{1}_{m^{\oOT{B}_{H+1}}}}{=\vect{b}.}
    \end{mini}
Set
\(
    \sd{D}_H=\oOT{A}\seqcomp (\oOT{B}_{11}\otimes \dots \otimes \oOT{B}_{1l_1})\seqcomp \dots \seqcomp (\oOT{B}_{H1}\otimes \dots \otimes \oOT{B}_{Hl_{H}}),
\)
and introduce a slack variable $\vect{u}$ to rewrite \eqref{opt:basePrimal_D_H+1} as     \begin{mini*}|l|
        {\substack{\vect{P}^{\oOT{A}},((\vect{P}^{\oOT{B}_{ji}})_{i=1}^{l_j})_{j=1}^{H+1},\vect{u}}} 
        {\langle \vect{C}^{\oOT{A}}, \vect{P}^{\oOT{A}} \rangle +  \sum\limits_{j,i}\langle \vect{C}^{\oOT{B}_{ji}}, \vect{P}^{\oOT{B}_{ji}}\rangle} 
        {\label{opt:slackPrimal_D_H+1}} 
        {} 
        \addConstraint{\vect{P}^{\oOT{A}}\vect{1}_{n^{\oOT{A}}} }{= \vect{a}} 
        \addConstraint{(\vect{P}^{\oOT{A}})^{\top}\vect{1}_{m^{\oOT{A}}}}{= (\vect{P}^{\oOT{B}_1})\vect{1}_{n^{\oOT{B}_1}}} 
        \addConstraint{(\vect{P}^{\oOT{B}_{j}})^{\top}\vect{1}_{m^{\oOT{B}_{j}}}}{= (\vect{P}^{\oOT{B}_{j+1}})\vect{1}_{n^{\oOT{B}_{j+1}}}}{(j\in[H-1])}
        \addConstraint{(\vect{P}^{\oOT{B}_{H}})^\top\vect{1}_{m^{\oOT{B}_{H}}}}{=\vect{u}}
        \addConstraint{(\vect{P}^{\oOT{B}_{H+1}})\vect{1}_{n^{\oOT{B}_{H+1}}}}{=\vect{u}}
        \addConstraint{(\vect{P}^{\oOT{B}_{H+1}})^\top\vect{1}_{m^{\oOT{B}_{H+1}}}}{=\vect{b},}
    \end{mini*}
    which is equivalent to
    \begin{mini*}|l|
        {\substack{(\vect{P}^{\oOT{B}_{(H+1)i}})_{i=1}^{l_{H+1}},\vect{u}}} 
        {\OT{\sd{D}_H}{\vect{a}}{\vect{u}} +  \sum_{i=1}^{l_{H+1}}\langle \vect{C}^{\oOT{B}_{(H+1)i}}, \vect{P}^{\oOT{B}_{(H+1)i}}\rangle} 
        {\label{opt:slackPrimal_D_H+1_2}} 
        {} 
        \addConstraint{(\vect{P}^{\oOT{B}_{H+1}})\vect{1}_{n^{\oOT{B}_{H+1}}}}{=\vect{u}}
        \addConstraint{(\vect{P}^{\oOT{B}_{H+1}})^\top\vect{1}_{m^{\oOT{B}_{H+1}}}}{=\vect{b}.}
    \end{mini*}
By the hypothesis of the induction, we can further rewrite the problem as
    \begin{mini*}|l|
        {\substack{\vect{P}^{\sd{D}_H},(\vect{P}^{\oOT{B}_{(H+1)i}})_{i=1}^{l_{H+1}}}} 
        {\left\langle\vect{C}^{\sd{D}_H},\vect{P}^{\sd{D}_H}\right\rangle +  \sum_{i=1}^{l_{H+1}}\langle \vect{C}^{\oOT{B}_{(H+1)i}}, \vect{P}^{\oOT{B}_{(H+1)i}}\rangle} 
        {} 
        {} 
        \addConstraint{(\vect{P}^{\sd{D}_{H}})\vect{1}_{n^{\oOT{A}}}}{=\vect{a}}
        \addConstraint{(\vect{P}^{\sd{D}_{H}})^\top\vect{1}_{m^{\oOT{B}_{H+1}}}}{=(\vect{P}^{\oOT{B}_{H+1}})\vect{1}_{n^{\oOT{B}_{H+1}}}}
        \addConstraint{(\vect{P}^{\oOT{B}_{H+1}})^\top\vect{1}_{m^{\oOT{B}_{H+1}}}}{=\vect{b}.}
    \end{mini*}
    In the same way as the base step {\bf{(I)}}, we can finally see that \eqref{opt:basePrimal_D_H+1} is equivalent to $\SeqOT{\vect{C}^{\sd{D}_H}}{\oOT{B}_{H1}\otimes \dots \otimes \oOT{B}_{Hl_{H}}}{\vect{a}}{\vect{b}}$, or $\OT{\vect{C}^{\sd{D}_H\seqcomp(\oOT{B}_{H1}\otimes \dots \otimes \oOT{B}_{Hl_{H}})}}{\vect{a}}{\vect{b}}$.
    Therefore, the dual of $\OT{\sd{D}_{H+1}}{\vect{a}}{\vect{b}}$ is given by the dual of $\OT{\vect{C}^{{\sd{D}_{H+1}}}}{\vect{a}}{\vect{b}}$.
\end{proof}

\section{Symmetric Strict Monoidal Category}
\label{appendix:SSMC}

We refer to~\cite {mac2013categories} as a reference. 
We recall \emph{symmetric strict monoidal categories (SSMC)}:
\begin{defn}[symmetric strict monoidal category]
    A \emph{symmetric strict monoidal category (SSMC)} is a category 
    $(\mathfrak{C}, \seqcomp)$ equipped with the \emph{unit} $0\in \ob{\mathfrak{C}}$,
    the bifunctor  $\otimes \colon\mathfrak{C}\times \mathfrak{C}\rightarrow \mathfrak{C}$,
    and the natural isomorphism $\sigma_{m, n}\colon m\otimes n \rightarrow n\otimes m$ for each objects $m, n\in \ob{\mathfrak{C}}$ such that the following 
    conditions hold:
    \begin{itemize}
        \item $l\otimes (m\otimes n) = (l\otimes m)\otimes n$,
        \item $0\otimes m = m = m\otimes 0$, 
        \item $\oOT{A}\otimes (\oOT{B}\otimes \oOT{C}) = (\oOT{A}\otimes \oOT{B})\otimes \oOT{C}$,
        \item $\id{0}\otimes \oOT{A} = \oOT{A} = \oOT{A}\otimes \id{0}$,
        \item $\sigma_{m, 0} = \id{m}$, $\sigma_{m, n}\seqcomp \sigma_{n, m} = \id{m\otimes n}$,
        \item $\sigma_{l, m\otimes n} = (\sigma_{l, m}\otimes \id{n})\seqcomp (\id{m}\otimes \sigma_{l, n})$,
    \end{itemize}
    for any objects $l, m, n\in \ob{\mathfrak{C}}$, and morphisms $\oOT{A}, \oOT{B}, \oOT{C}$. 
\end{defn}

\begin{proof}[Proof of~\cref{prop:SSMC}]
It has been well-known that matrices over a commutative semiring form a \emph{PROP}~\cite{maclane1965categorical,Zanasi15}, which is a subclass of SSMCs.
Since cost matrices are composed over the min-tropical semiring, we can conclude that they form an SSMC. 
    
\end{proof}

\begin{lemma} 
    \label{lem:bifunctoriality}
    For any morphisms $\oOT{A}\colon m_1 \rightarrow l_1$, 
    $\oOT{B}\colon l_1 \rightarrow n_1$, $\oOT{C}\colon m_2 \rightarrow l_2$, $\oOT{D}\colon l_2 \rightarrow n_2$, 
    the following equation holds: 
    \begin{align*}
        (\oOT{A}\seqcomp \oOT{B})\otimes (\oOT{C}\seqcomp \oOT{D}) = (\oOT{A}\otimes \oOT{C})\seqcomp (\oOT{B}\otimes \oOT{D}). 
    \end{align*}   
\end{lemma}
\begin{proof}
By definition of the bifunctor $\otimes$. 
\end{proof}

\begin{proof}[Proof of~\cref{prop:SDToCSD}]
    We present a reduction from a given string diagram $\sd{D}_1$ to 
    to a string diagram $\sd{D}_2$ that is sequential normal form.
    We define the reduction by the structural induction.  
    \paragraph{(I) Base case: $\sd{D}_1 = \oOT{A}$:} The string diagram $\sd{D}_1$ is already sequential normal form. 
    \paragraph{(II) Base case: $\sd{D}_1 = (n, n, \idcost{n})$:} The string diagram $\sd{D}_1$ is already sequential normal form. 
    \paragraph{(III) Step case: $\sd{D}_1 = \sd{D}\seqcomp \sd{D}'$:} By the hypothesis of the induction, 
    we assume that $\sd{D}$ and $\sd{D}'$ are sequential normal form. Then, $\sd{D}_1$ is also in sequential normal form. 
    \paragraph{(IV) Step case: $\sd{D}_1 = \sd{D}\otimes \sd{D}'$:} By the hypothesis of the induction, we assume that $\sd{D}$ and $\sd{D}'$ are sequential normal form. 
    Suppose that $\sd{D} \defeq (\oOT{B}_{11}\otimes \dots \otimes \oOT{B}_{1l_1})\seqcomp \dots \seqcomp (\oOT{B}_{H1}\otimes \dots \otimes \oOT{B}_{Hl_H})$ and $\sd{D}' \defeq (\oOT{B}'_{11}\otimes \dots \otimes \oOT{B}'_{1l'_1})\seqcomp \dots \seqcomp (\oOT{B}'_{H'1}\otimes \dots \otimes \oOT{B}'_{H'l'_{H'}})$. 
    We can assume that $H = H'$ because we can adjust the length by sequentially composing with identities. 
    Now, let $\sd{D} = (\oOT{B}_{11}\otimes \dots \otimes \oOT{B}_{1l_1})\seqcomp \sd{E}$ and $\sd{D}' = (\oOT{B}'_{11}\otimes \dots \otimes \oOT{B}_{1l'_1})\seqcomp \sd{E}'$. 
    By~\cref{lem:bifunctoriality}, we can say that the cost matrices $\sd{D}\otimes \sd{D}'$ and $ (\oOT{B}_{11}\otimes \dots \otimes \oOT{B}_{1l_1}\otimes \oOT{B}'_{11}\otimes \dots \otimes \oOT{B}_{1l'_1})\seqcomp (\sd{E}\otimes \sd{E}')$ are the same. 
    By continuing this process, we can get the sequential normal form $(\oOT{B}_{11}\otimes \dots \otimes \oOT{B}'_{1l'_1})\seqcomp \dots \seqcomp (\oOT{B}_{H1}\otimes \dots \otimes \oOT{B}'_{H'l'_{H'}})$. 

    By definition, the reduction preserves the cost matrix, concluding that the string diagram $\sd{D}_2$ that is sequential normal form constructed by the reduction satisfies that  $\vect{C}^{\oOT{A}\seqcomp \sd{D}_1} = \vect{C}^{\oOT{A}\seqcomp\sd{D}_2}$. 
\end{proof}

We remark that a simpler and stronger statement of~\cref{prop:SDToCSD} holds by the same argument if we rely on category theory.  
In fact,  the proof of~\cref{prop:SDToCSD} also proves the following stronger statement; see e.g.~\cite{Zanasi15} for a reference of \emph{free PROPs} generated by \emph{symmetric monoidal theories}. 
\begin{prop} 
    Let $\mathfrak{C}$ be the free PROP generated by oOTs.
    Given an oOT $\oOT{A}\colon m\rightarrow l$ and a string diagram $\sd{D}_1\colon l\rightarrow n$, 
    there is an aligned string diagram $\oOT{A}\seqcomp \sd{D}_2$ such that (i) $\sd{D}_1 = \sd{D}_2$ holds in $\mathfrak{C}$; 
    (ii) $\sd{D}_2$ is sequential normal form; (iii) the equation $\vect{C}^{\oOT{A}\seqcomp \sd{D}_1} = \vect{C}^{\oOT{A}\seqcomp\sd{D}_2}$ holds. 
\end{prop}

\section{Relaxation}
\label{sec:relaxation}

We show an example such that the optimal value~\cref{def:relaxOtsChoices} is strictly greater than the one of~\cref{def:otsChoices}. 
\begin{itemize}
\item the two distributions are the uniform distributions $\vect{a} = \vect{b} = (1/3, 1/3, 1/3)$. 
\item the OT is given by $\OT{\oOT{A}_1\seqcomp \oOT{A}_2}{\vect{a}}{\vect{b}}$ where $\oOT{A}_1$ and $\oOT{A}_2$ share the same set of cost matrices $\{\vect{C}_1, \vect{C}_2\}$ given by
\begin{align*}
    \vect{C}_1 &\defeq\begin{pmatrix}
            15 & 12  & 4 \\
            9 &  6   & 10 \\ 
            4 &  9   & 14
        \end{pmatrix},
    &\vect{C}_2 \defeq\begin{pmatrix}
            6  & 12   & 5 \\
            1  &  4   & 7 \\ 
            17 & 11   & 12
        \end{pmatrix}.
\end{align*}
\end{itemize}

\section{Benchmark Details}
\label{appendix:benchmarkDetails}

We show the details of the four benchmarks $\texttt{BRooms}$, $\texttt{URooms}$, $\texttt{BChains}$, $\texttt{UChains}$. 
All cost matrices in benchmarks are randomly generated, whose elements $C_{ij}$ are  $C_{ij}\in [0, 10^6]$.
String diagrams of each benchmark are illustrated as follows. 
\begin{description}
    \item[$\texttt{BRooms}$:] $\oOT{A}\seqcomp (\oOT{B}_{11}\otimes \dots \otimes \oOT{B}_{1l_1})\seqcomp \dots \seqcomp (\oOT{B}_{H1}\otimes \dots \otimes \oOT{B}_{Hl_H})\seqcomp \oOT{C}$, 
    \item[$\texttt{URooms}$:] $\oOT{A}\seqcomp (\oOT{B}_{11}\otimes \dots \otimes \oOT{B}_{1l_1})\seqcomp (\oOT{C}_{11}\otimes \dots \otimes \oOT{C}_{1l_1})\seqcomp \dots \seqcomp (\oOT{B}_{H1}\otimes \dots \otimes \oOT{B}_{Hl_{H}})\seqcomp \oOT{D}$, 
    \item[$\texttt{BChains}$:] $\oOT{A}_1\seqcomp\dots \seqcomp \oOT{A}_H $
    \item[$\texttt{UChains}$:] $\oOT{A}_1\seqcomp \oOT{B}_1\dots \seqcomp \oOT{A}_H \seqcomp \oOT{B}_H$
\end{description}
Each benchmark has multiple string diagrams by changing cost matrices and their sizes.  
We set two distributions $\vect{a} \in \Delta^m$ and $\vect{b} \in \Delta^n$ as uniform distributions. 


The benchmarks $\texttt{BRooms}$ and $\texttt{BChains}$ are \emph{balanced}:
for each string diagram in $\texttt{BRooms}$, the sizes of $(\vect{C}^{\oOT{B}_{11}}\otimes \dots \otimes \vect{C}^{\oOT{B}_{1l_1}}),\dots ,(\vect{C}^{\oOT{B}_{H1}}\otimes \dots \otimes \vect{C}^{\oOT{B}_{Hl_H}})$ are all same and $N\times N$ matrices for some $N\in \N$, and 
for each string diagram in $\texttt{BChains}$, the sizes of $\vect{C}^{\oOT{A}_1},\dots, \vect{C}^{\oOT{A}_H}$ are all same as well. 

On the other hand, the benchmarks $\texttt{URooms}$ and $\texttt{UChains}$ are \emph{unbalanced}:
for each string diagram in $\texttt{URooms}$, the sizes of $(\vect{C}^{\oOT{B}_{i1}}\otimes \dots \otimes \vect{C}^{\oOT{B}_{il_i}})$ and $(\vect{C}^{\oOT{C}_{i1}}\otimes \dots \otimes \vect{C}^{\oOT{C}_{il_i}})$ are different and $M\times N$ and $N\times M$ matrices for some $M, N\in \N$, and 
for each string diagram in $\texttt{Uhains}$, the sizes of $\vect{C}^{\oOT{A}_i}$ and $\vect{C}^{\oOT{B}_i}$ are different as well. 

We show the further details for each benchmarks shown in~\cref{tab:runtimesetc} and~\cref{fig:infAlgSt}. 
\subsection{Full Detail of The Benchmarks in \texorpdfstring{\cref{tab:runtimesetc}}{Table:runtimesetc}}
\paragraph{\texttt{BRoom1}.} The string diagram is given by  $\oOT{A}\seqcomp (\oOT{B}_{1,1}\otimes \oOT{B}_{1,2})\seqcomp \dots \seqcomp (\oOT{B}_{99,1} \otimes \oOT{B}_{99,2})\seqcomp \oOT{C}\colon 100\rightarrow 100$, where 
$\oOT{A}\colon 100\rightarrow 100$, $\oOT{C}\colon 100\rightarrow 100$, and $\oOT{B}_{i,j}$ is given by (i) $\oOT{B}_{i,1}\colon 30 \rightarrow 30$ if $i$ is even; (ii) $\oOT{B}_{i,2}\colon 70 \rightarrow 70$ if $i$ is even; (iii) $\oOT{B}_{i,1}\colon 40 \rightarrow 40$ if $i$ is odd; (iv) $\oOT{B}_{i,2}\colon 60 \rightarrow 60$ if $i$ is odd.

\paragraph{\texttt{BRoom2}.} The string diagram is given by  $\oOT{A}\seqcomp (\oOT{B}_{1,1}\otimes \dots \otimes \oOT{B}_{1,208})\seqcomp \oOT{C}\colon 100\rightarrow 100$, where 
$\oOT{A}\colon 100\rightarrow 20800$, $\oOT{C}\colon 20800\rightarrow 100$, and $\oOT{B}_{1,j}\colon 100\rightarrow 100$ for all $j$. 

\paragraph{\texttt{URoom1}.} The string diagram is given by  
\begin{align*}
    \oOT{A}\seqcomp (\oOT{B}_{1,1}\otimes \oOT{B}_{1,2})\seqcomp (\oOT{C}_{1,1} \otimes \oOT{C}_{1,2})\seqcomp \dots \seqcomp (\oOT{B}_{100,1}\otimes  \oOT{B}_{100,2})\seqcomp \oOT{D}\colon 10\rightarrow 10,
\end{align*}
where $\oOT{A}\colon 10\rightarrow 500$, $\oOT{C}\colon 10\rightarrow 10$, and $\oOT{B}_{i,j}$ is given by (i) $\oOT{B}_{i,1}\colon 270 \rightarrow 3$; (ii) $\oOT{B}_{i,2}\colon 230 \rightarrow 7$; (iii) $\oOT{C}_{i,1}\colon 4 \rightarrow 240$; (iv) $\oOT{C}_{i,2}\colon 6 \rightarrow 260$, for all $i$.

\paragraph{\texttt{URoom2}.} The string diagram is given by  
\begin{align*}
    \oOT{A}\seqcomp (\oOT{B}_{1,1}\otimes  \oOT{B}_{1,2})\seqcomp (\oOT{C}_{1,1}\otimes \oOT{C}_{1,2})\seqcomp \dots \seqcomp (\oOT{B}_{150,1}\otimes  \oOT{B}_{150,2})\seqcomp \oOT{D}\colon 10\rightarrow 10,
\end{align*}
where $\oOT{A}\colon 10\rightarrow 500$, $\oOT{C}\colon 10\rightarrow 10$, and $\oOT{B}_{i,j}$ is given by (i) $\oOT{B}_{i,1}\colon 270 \rightarrow 3$; (ii) $\oOT{B}_{i,2}\colon 230 \rightarrow 7$; (iii) $\oOT{C}_{i,1}\colon 4 \rightarrow 240$; (iv) $\oOT{C}_{i,2}\colon 6 \rightarrow 260$, for all $i$.

\paragraph{\texttt{BChain1}.} The string diagram is given by $\oOT{A}_1\seqcomp\dots \seqcomp \oOT{A}_{210}$, where $\oOT{A}_i\colon 100\rightarrow 100$ for all $i$.  

\paragraph{\texttt{BChain2}.} The string diagram is given by $\oOT{A}_1\seqcomp\dots \seqcomp \oOT{A}_{400}$, where $\oOT{A}_i\colon 100\rightarrow 100$ for all $i$.

\paragraph{\texttt{UChain1}.} The string diagram is given by $\oOT{A}_1\seqcomp \oOT{B}_1\dots \seqcomp \oOT{A}_{199} \seqcomp \oOT{B}_{199}\seqcomp \oOT{A}_{200}$, where $\oOT{A}_i\colon 10\rightarrow 200$ and $\oOT{B}_i\colon 200\rightarrow 10$ for all $i$.  

\paragraph{\texttt{UChain2}.} The string diagram is given by $\oOT{A}_1\seqcomp \oOT{B}_1\dots \seqcomp \oOT{A}_{399} \seqcomp \oOT{B}_{399}\seqcomp \oOT{A}_{400}$, where $\oOT{A}_i\colon 10\rightarrow 200$ and $\oOT{B}_i\colon 200\rightarrow 10$ for all $i$.

\subsection{Full Detail of The Benchmarks in \texorpdfstring{\cref{fig:infAlgSt}}{Fig:infAlgSt}}
\paragraph{Sequential Compositions (the left figure).} We run the experiment with a set of \texttt{BChains}  $\oOT{A}_1\seqcomp\dots \seqcomp \oOT{A}_{H}$, where $\oOT{A}_i\colon 100\rightarrow 100$ for all $i$. 
The x-axis represents the parameter $H$: we set $H = 100, 200, 300, 400, 500, 600, 700$. 

\paragraph{Parallel Compositions (the right figure).} We run the experiment with a set of \texttt{BRooms} $\oOT{A}\seqcomp (\oOT{B}_{1,1}\otimes \dots \otimes \oOT{B}_{1,H})\seqcomp \oOT{C}\colon 100\rightarrow 100$, where $\oOT{A}_i\colon 100\rightarrow 100H$, $\oOT{C}\colon 100H\rightarrow 100$, and $\oOT{B}_{1, i}\colon 100\rightarrow 100$ for all $i$. 
The x-axis represents the parameter $H$: we set $H = 28, 58, 88, 118, 148, 178, 208$.

\else 
\fi

\end{document}